\documentclass{article}

    \PassOptionsToPackage{numbers, compress}{natbib}

 \usepackage[main, final]{neurips_2026}

\usepackage[utf8]{inputenc} 
\usepackage[T1]{fontenc}    
\usepackage{hyperref}       
\usepackage{url}            
\usepackage{booktabs}       
\usepackage{amsfonts}       
\usepackage{nicefrac}       
\usepackage{microtype}      
\usepackage[dvipsnames,table]{xcolor} 

\usepackage{amsmath}
\usepackage{amssymb}
\usepackage{mathtools}
\usepackage{amsthm}

\usepackage[capitalize,noabbrev]{cleveref}

\theoremstyle{plain}
\newtheorem{theorem}{Theorem}[section]
\newtheorem{proposition}[theorem]{Proposition}

\newtheorem{corollary}[theorem]{Corollary}
\theoremstyle{definition}
\newtheorem{definition}[theorem]{Definition}

\theoremstyle{remark}

\usepackage{subcaption}
\usepackage{colortbl}
\usepackage{adjustbox}
\usepackage{titletoc}
\usepackage{multirow}

\title{Contrastive Representation Shaping for LLM Unlearning}

\author{Haoran Tang \\
  Department of Computer Science \\
  Purdue University \\
  \texttt{thr@purdue.edu} \\
  \And
  Rajiv Khanna \\
  Department of Computer Science \\
  Purdue University \\
  \texttt{rajivak@purdue.edu} \\
}

\begin{document}

\maketitle

\begin{abstract}
    Most LLM unlearning methods aim to approximate retrain-from-scratch behaviors with minimal distribution shift, often via alignment-style objectives defined in the prediction space. While effective at reducing forgotten content generation, such approaches may act as suppression: forgotten concepts can persist in representations and remain entangled with retained knowledge. We introduce CLReg, a contrastive representation regularizer that identifies forget features while pushing them away from retain features, reducing forget--retain interference while empirically preserving the scale and shape of retain features. As light motivation for the mechanism, we provide a one-step analysis showing that CLReg decreases a simple entanglement proxy in the embedding space. Across unlearning benchmarks and LLMs of different sizes, CLReg decreases forget-retain representation entanglement to enhance mainstream unlearning methods without extra privacy risks, inspiring future unlearning work to remove forget concepts via representation shaping. Code is available at \url{https://github.com/HaoranTang/CLReg}.
\end{abstract}
\vspace{-10pt}
\section{Introduction}
\vspace{-5pt}
Removing the influence of specific training data after deployment—\emph{machine unlearning}—is increasingly important for privacy compliance and model maintenance. Large language models (LLMs) sharpen this need: they memorize and regenerate verbatim sequences from training corpora, so objectionable or proprietary content must be deletable on demand. Retraining from scratch on retained data is the gold-standard, but its cost is prohibitive for modern LLMs, motivating a surge of \emph{approximate unlearning} methods that aim to mimic a retrained model at a fraction of the cost.

Early methods range from fine-tuning on the retain set and gradient ascent on the forget set to student--teacher distillation and saliency-based weight masking~\cite{kurmanji2023towards,fan2024salun}. While effective on small models, they often degrade utility and require careful tuning. Recent investigations show that unlearning becomes harder as forget and retain distributions grow more \emph{entangled} or as forget examples are heavily memorized~\cite{zhao2024makes}, so that separating these representations becomes central to selective forgetting.

Unlearning in LLMs is particularly challenging: naively maximizing loss on the forget set leads to instability and catastrophic degradation of general capabilities~\cite{zhang2024negative,yao2024large}. Recent alignment-based methods dominate the literature---Negative Preference Optimization (NPO)~\cite{zhang2024negative} reweights the forgetting objective to discourage forgotten content while preserving utility; SimNPO~\cite{fan2024simplicity} simplifies it to remove reference-model bias; others use self-distillation with adjusted logits~\cite{dong2024undial} or primal--dual constrained entropic unlearning~\cite{entesari2025constrained}. Despite their differences, these methods share a single philosophy: align the unlearned model's \emph{prediction distribution} with the retrained model's, treating any deviation in representation space as undesirable. They succeed at suppressing forgotten outputs, but suppression need not remove the underlying knowledge: forgotten concepts can persist in the representation space, entangled with retained ones. Indeed, a relearning attack of only 32 fine-tuning steps closes over half of the gap back to the fine-tuned model for SimNPO (Sec.~\ref{sec:robust}).

An emerging view is that \emph{representation shaping} can address this limitation. Empirical evidence shows that unlearning difficulty correlates with the degree of forget--retain entanglement~\cite{zhao2024makes,tang2025sharpness}: separating these clusters should make it easier to erase one without distorting the other. In representation learning, contrastive objectives are known to simultaneously \emph{align} similar examples and \emph{disperse} representations on the hypersphere~\cite{wang2020understanding}, with simple augmentations and a cosine-similarity loss yielding tight positive clusters and uniform negatives~\cite{chen2020simple,gao2021simcse}. Building on this, contrastive unlearning has been proposed for small classifiers: \citet{ijcai2025p830} push forget embeddings away from their original class toward alternative regions, and \citet{khalil2025coun} align forget examples with retain semantics. These works suggest that explicitly shaping the feature space can make forgetting more targeted, but they operate in \textbf{supervised} classification settings with modest model sizes and rely on clean labels for positives and negatives. Whether similar benefits transfer to generative LLMs with self-supervision---where forget targets are instance-specific and the representation space is high-dimensional---has remained unclear.

We propose \emph{contrastive representation regularization} (\textbf{CLReg}). We treat the \emph{forget concept} as the coherent feature subspace spanned by embeddings of forget examples (Def.~\ref{def:concept}), rather than per-example memorization targets. Our key idea is to isolate this subspace and push it away from retain features, reducing entanglement while minimally perturbing the retain representation. We construct positive pairs for forget examples via lightweight augmentations (dropout masks and paraphrases), treat retain embeddings as negatives, and apply a contrastive loss that clusters forget anchors with their own augmentations and repels them from retain features. CLReg integrates with existing unlearning algorithms across families. As a light mechanistic motivation, we show that a single contrastive update strictly decreases anchor--negative similarity and a simple distributional entanglement proxy, giving intuition for why CLReg reduces forget--retain entanglement in practice. Empirically, across multiple benchmarks and LLMs, CLReg reduces entanglement and improves forget quality on top of state-of-the-art unlearning methods without introducing privacy risk. These findings challenge the prevailing belief that representation distributions must stay close to a retrained model: explicit representation shaping can \emph{help} unlearning, and the clear forget/retain separation visualized in our results points to surgical latent-space interventions as a promising future direction.

Our contributions can be summarized as follows:

\textbf{Rethinking representation shaping}: Our study effectively shows that regularizing forget concepts in the representation space will not deviate the goal of unlearning or lead to collapse. Instead, our CLReg effectively separates forget features with minimal shifts on retain features, and improves unlearning performance with little privacy concerns.

\textbf{Empirical analysis on entanglement with light mechanistic motivation}: We formalize the forget concept as a feature subspace (Def.~\ref{def:concept}) and provide a simple one-step analysis in the embedding space that explains \emph{why} CLReg reduces forget--retain entanglement; we present this as motivation and intuition rather than a standalone theoretical contribution. Our empirical evidence is where the entanglement claim is substantiated: quantitative entanglement metrics (variance-based entanglement, MK-MMD, $W_2$) and UMAP visualizations consistently show that CLReg pushes away forget features while keeping retain features close to their original scale and shape. Beyond observing this effect, we intervene on it: shaping representations with no forgetting objective at all improves the subsequent unlearning of an unmodified base method, and CLReg makes forgotten knowledge harder to recover under a relearning attack.

\textbf{CLReg as a plug-and-play regularizer}: We propose CLReg as a plug-in regularizer compatible with a wide range of unlearning objectives. We adapt contrastive construction to the generative, self-supervised LLM unlearning regime -- where clean class labels for positives/negatives are unavailable -- through dropout and paraphrase augmentations on forget anchors, with optional DPO-style, symmetric, or InfoNCE variants.

\textbf{Empirical validation}: Across TOFU, MUSE and WMDP on four LLMs (Llama-3-8B/3B, Llama-2-7B, Zephyr-7B-$\beta$), CLReg improves the \texttt{UnlearningScore} of all five mainstream unlearning methods (GradDiff, UNDIAL, NPO, SimNPO, PDU) in all 25 method--benchmark pairs. SimNPO+CL achieves the best \texttt{UnlearningScore} on every benchmark, beating the strongest CLReg-free baseline by $+0.039$ on TOFU-8B ($0.774\!\to\!0.813$, vs.\ $0.704$ for the representation-shaping baseline RMU) and improving SimNPO by $+0.110$ on MUSE-News ($0.461\!\to\!0.572$); GradDiff+CL delivers the largest single-method gain, lifting \texttt{UnlearningScore} from $0.411$ to $0.738$ on MUSE-Books. CLReg cuts forget--retain entanglement by up to $88\%$ (SimNPO on MUSE-News, $92.2\!\to\!11.1$) without degrading retain-side utility, and pulls $|\texttt{PrivLeak}|$ closer to zero in most settings (e.g., $62.0\!\to\!3.7$ for SimNPO+CL on TOFU-3B). Ablations on contrastive variants, the number and selection of negatives, layer choice, and runtime confirm the gains are robust to design choices, and CLReg preserves general capability and sharply reduces utility erosion under sequential unlearning requests.
\section{Related Work}
\vspace{-5pt}
\textbf{Foundations in unlearning.} Early \emph{approximate} unlearning methods aim to erase designated forget examples without retraining. A common baseline fine-tunes on retain data and relies on catastrophic forgetting~\cite{golatkar2020eternal,warnecke2021machine}; more direct approaches perform \emph{gradient ascent} on the forget set, which actively degrades memory but can harm utility. NegGrad+ balances the two by jointly maximizing forget loss and minimizing retain loss, with student--teacher training following the same scheme~\cite{kurmanji2023towards}. SalUn instead restricts updates to weights most responsible for forget-set predictions, limiting collateral damage~\cite{fan2024salun}. Other paradigms include label remixing and second-order influence estimates~\cite{graves2021amnesiac,izzo2021approximate}.

\textbf{LLM unlearning.} For LLMs, gradient-difference objectives (GradDiff) counteract the forget set while preserving retained behavior~\cite{maini2024tofu}. To improve stability, NPO~\cite{zhang2024negative} reframes forgetting as preference optimization to control dynamics; SimNPO~\cite{fan2024simplicity} removes a reference-model bias that favors easy-to-forget instances. UnDIAL~\cite{dong2024undial} replaces explicit loss maximization with a distillation-style objective to prevent collapse, and \citet{entesari2025constrained} cast unlearning as a primal--dual constrained problem for better Pareto trade-offs. A parallel line studies losses and data strategies beyond the standard forget--retain split~\cite{yao2024large,wang2025forgetonly,li2026llmbeliefs} (see \citet{liu2025rethinking} for a broader survey), but operates at the logit or loss level. \citet{li2024wmdp} introduce the WMDP benchmark together with RMU, which unlearns by \emph{misdirecting} forget-layer activations toward a random direction---an early representation-level intervention that validates the premise but differs fundamentally from CLReg's \emph{structured} contrastive separation. CLReg is complementary: it shapes the representation space rather than the output distribution and can be layered on top of logit-level objectives.

\textbf{Contrastive learning for unlearning.} Contrastive objectives offer a complementary path that \emph{explicitly shapes representations} to reduce forget--retain entanglement. \citet{ijcai2025p830} push forget embeddings away from their original class clusters toward alternative regions, and CoUn~\cite{khalil2025coun} restructures the latent space so that forget examples align with retain-induced semantics. Both demonstrate that representation-space restructuring sharpens forgetting on \emph{supervised}, small-to-medium classifiers, where class labels supply a natural contrastive signal. Extending this to LLM unlearning is non-trivial: forget targets are often instance- or concept-level without clean labels, the model is vastly larger, and the success criterion is \emph{behavioral} equivalence to a retain-only retrained model rather than mere representation separation. CLReg fills this gap by isolating and pushing away forget features from retain features in LLM representations.

\textbf{Theoretical analyses of contrastive learning.} \citet{arora2019theoretical} derive downstream guarantees for contrastive learning under a latent-class assumption, and \citet{wang2020understanding} characterize InfoNCE through \emph{alignment} of positive pairs and \emph{uniformity} on the hypersphere; both analyze representation-level proxies rather than end-to-end task guarantees. Our mechanistic motivation (Sec.~\ref{sec:motivation}) follows this methodology, using a one-step decrease in anchor--negative similarity as a transparent proxy.

\section{CLReg: Contrastive Regularization}
\vspace{-5pt}
\subsection{Preliminaries}
\vspace{-5pt}
Let $\mathbf{W}_{\text{FT}}$ denote the finetuned model, $\mathbf{W}_{\text{RT}}$ the retrained (target) model trained from scratch on the retain set, and $\mathbf{W}_{\text{UL}}$ the model undergoing unlearning.  We have access to an original training set $\mathcal{S}$ drawn from distribution $\mathcal{D}$, which is later decomposed into a \emph{retain} set $\mathcal{R}$ and a \emph{forget} set $\mathcal{F}$ after finetuning, with $\mathcal{R}=\mathcal{S}\setminus\mathcal{F}$ and $|\mathcal{R}|>|\mathcal{F}|$.  For any model $\mathbf{W}$, we denote by $h_\mathbf{W}(x)\in\mathbb{R}^{T\times d}$ the hidden-state matrix of $\mathbf{W}$ for input $x$ with $T$ tokens and $d$ hidden dimension.  We define $\mathrm{Pool}(h,m)=\sum_t m_t h_t/ \sum_t m_t$ to be mean pooling of hidden states with attention mask $m$ and token position index $t$.  For brevity, we write $\zeta_\mathbf{W}(x)= \mathrm{norm}\!\bigl(\mathrm{Pool}(h_\mathbf{W}(x),m(x))\bigr)\in \mathbb{S}^{d-1}$ for the $\ell_2$–normalized embedding (unit sphere) and use cosine similarity $s(u,v)=u^\top v$.

\begin{definition}[Forget concept]
\label{def:concept}
Given a forget set $\mathcal{F}$ and a model with parameters $\theta$, we define the \emph{forget concept} as the coherent feature subspace
\[
\mathcal{C}_{\mathcal{F}}^{\theta} = \mathrm{span}\bigl\{\zeta_{\theta}(x)\;\big|\;x\in\mathcal{F}\bigr\}\subseteq\mathbb{R}^{d},
\]
i.e.\ the subspace spanned by the pooled, normalized embeddings of forget inputs in the latent space.
\end{definition}
This definition makes the object of unlearning explicit: CLReg does not target the verbatim reproduction of individual sequences, but the \emph{feature cluster} $\mathcal{C}_{\mathcal{F}}^{\theta}$ that carries the high-level knowledge associated with the forget set (e.g., author biographies in TOFU, book/news content in MUSE, hazardous domain expertise in WMDP). Correspondingly, our target is to reduce the entanglement between $\mathcal{C}_{\mathcal{F}}^{\theta}$ and the retain subspace $\mathcal{C}_{\mathcal{R}}^{\theta}$, so that downstream unlearning objectives can remove the forget-side features with less collateral damage to retain features. Per-example augmentations (dropout, paraphrase) serve as proxies that probe this subspace: samples from different augmentations should reside in $\mathcal{C}_{\mathcal{F}}^{\theta}$, whereas retain embeddings lie outside it and can be used as negatives.

\subsection{Separating Forget Concepts}
\vspace{-5pt}
We draw inspiration from recent advances in contrastive representation learning that emphasize \emph{alignment} (bringing positive pairs close) and \emph{uniformity} (spreading all representations)~\cite{wang2020understanding}.  For each forget example $x_f\in\mathcal{F}$ with paraphrase $x_f^+$ we construct a positive pair $(z_f,z_f^+)$ via light augmentations:
\[
\begin{aligned}
    z_f &= \mathrm{Pool}(\mathrm{Dropout}(h_{\text{UL}}(x_f), p),m(x_f)), \\
    z_f^{+} &= \mathrm{Pool}(\mathrm{Dropout}(h_{\text{UL}}(x_f^+), p^\prime),m(x_f^+)).
\end{aligned}
\]
Here $p,p'\sim\mathcal{N}(\mu,\sigma)$ are independently sampled dropout rates which help differentiate $z_f$ from $z_f^{+}$~\cite{gao2021simcse}. In practice we pick $\mathcal{N}(0.1,0.05)$ and clamp $p$ to $[0,0.2]$ to add randomness as data augmentation. The paraphrase $x_f^+$ is precomputed once by a language model (Supp.~\ref{supp:impl}); if paraphrasing is unavailable we simply set $x_f^+=x_f$, so that the two views differ only in their dropout masks. Each retain item $x_r\in\mathcal{R}$ provides a negative $z_r^- = \mathrm{Pool}(h_{\text{UL}}(x_r),m(x_r))$ (without dropout). A \emph{DPO-style contrastive loss}~\cite{rafailov2023direct,xu2024contrastive} encourages every positive pair to have higher similarity than each of its negative pairs:
\begin{equation}
    \mathcal{L}_{\text{CL}}^{\text{dpo}} = -\frac{2\tau}{B_fB_r}\sum_{i=1}^{B_f}\sum_{j=1}^{B_r}\log \sigma \left(\frac{s(z_{f_i}, z_{f_i}^+) - s(z_{f_i}, z_{r_j}^-)}{\tau}\right),
    \label{eq:dpo-cl}
\end{equation}
where $\sigma$ is the sigmoid, $B_f$ and $B_r$ are the numbers of forget anchors and retain negatives, and the temperature $\tau>0$ controls hardness of negatives. Every anchor is contrasted with all $B_r$ negatives, gathered across data-parallel workers; this gives $32$ negatives per anchor in all our experiments (Supp.~\ref{supp:impl}).  Alternatively, a standard \emph{InfoNCE loss} $\mathcal{L}_{\text{CL}}^{\text{info}}$~\cite{oord2018representation} can be used: the logits concatenate the positive similarity and all cross-retain similarities, and a cross-entropy loss identifies the positive as the correct one. Both forms can be symmetrized by swapping anchor/negative roles (retain vs.\ forget).

\textbf{Combined objective.} Our CLReg acts as a regularizer on top of any base unlearning algorithm.  Suppose $\mathcal{L}_{\text{forget}}$ and $\mathcal{L}_{\text{retain}}$ denote the forget and retain losses (e.g.\ SimNPO for $\mathcal{L}_{\text{forget}}$ and cross-entropy for $\mathcal{L}_{\text{retain}}$).  We minimize
\begin{equation}
    \mathcal{L} = \alpha\mathcal{L}_{\text{retain}} + \gamma\mathcal{L}_{\text{forget}} + \lambda\mathcal{L}_{\text{CL}},
\end{equation}
with hyperparameters $\alpha,\gamma,\lambda\ge 0$. The CL term shapes the representation space while $\mathcal{L}_{\text{forget}}$ unlearns $\mathcal{F}$ and $\mathcal{L}_{\text{retain}}$ preserves $\mathcal{R}$. On current benchmarks the forget set consists of high-level semantic knowledge that prior work localizes to \emph{later} transformer layers~\cite{hong-etal-2024-dissecting}, so we apply CLReg there; Sec.~\ref{sec:which-layer} confirms that regularizing earlier layers degrades utility more than it helps forgetting. For concepts localized earlier (e.g., stylistic or syntactic patterns), concept-dependent layer selection is a natural direction for future work.

\subsection{A Mechanistic View through Alignment, Uniformity, and Latent-Class Separation}
\label{sec:motivation}
\vspace{-5pt}
We recast two established lenses on contrastive learning to the unlearning setting, and work in embedding space where ``entanglement'' is measured and where both lenses operate. \citet{wang2020understanding} decompose the contrastive loss into \emph{alignment} (pulling positives together) and \emph{uniformity} (spreading representations on the hypersphere); in CLReg, positives are augmented views of the same forget example, while the ``spread'' is restricted to forget-vs-retain pairs, yielding a \emph{directional uniformity} that pushes forget mass away from retain mass rather than spreading the entire batch. \citet{arora2019theoretical} analyze contrastive learning through unobserved latent classes and show that the population loss is minimized when class-conditional means are separated; identifying $\mathcal{C}_{\mathcal{F}}^{\theta}$ with one latent class and the retain distribution with the other makes our entanglement reduction a class-separation objective in their sense. Propositions~\ref{prop:cl}--\ref{prop:separate} make this concrete at the level of a single update: the embedding-space step is the directional-uniformity component of Wang \& Isola, applied across the (forget, retain) class pair of Arora et al.
\begin{definition}
\label{def:entanglement}
    \textbf{(Distributional separation and entanglement)}. Let $\zeta_\theta(x)\in\mathbb{R}^{d}$ denote the embedding of an input $x$ under parameters $\theta$ (pooled hidden states and normalized). Let
    \[
    \mathcal{P}_{\mathcal{F}}^\theta=\mathrm{Law}\bigl(\zeta_\theta(x)\;|\;x\in\mathcal{F}\bigr), \mathcal{P}_{\mathcal{R}}^\theta= \mathrm{Law}\bigl(\zeta_\theta(x)\;|\;x\in\mathcal{R}\bigr),
    \]
    denote the induced distributions of embeddings from forget and retain sets.  A \emph{separation measure} between $(\mathcal{P}_{\mathcal{F}}^\theta,\mathcal{P}_{\mathcal{R}}^\theta)$ is any probability metric $D$ such that $D(\mathcal{P}_{\mathcal{F}}^\theta,\mathcal{P}_{\mathcal{R}}^\theta)=0$ if and only if $\mathcal{P}_{\mathcal{F}}^\theta=\mathcal{P}_{\mathcal{R}}^\theta$. High separation $D(\mathcal{P}_{\mathcal{F}}^\theta,\mathcal{P}_{\mathcal{R}}^\theta)$ corresponds to low entanglement, whereas low separation (or high overlap) indicates that the representations of $\mathcal{F}$ and $\mathcal{R}$ are intertwined.
\end{definition}
\begin{proposition}
\label{prop:cl}
    \textbf{(Anchor update for DPO-CL)}. Consider a single $\mathcal{L}_{\text{CL}}^{\text{dpo}}$ term with $i$-th forget sample and $j$-th retain sample
    \[
    \ell_{ij}(\theta)
    = -2\tau \log \sigma\!\Bigl(\tfrac{m_{ij}}{\tau}\Bigr),
    m_{ij} = s(a_i,p_i) - s(a_i,n_j),
    \]
    where $a_i=\zeta_\theta(x_{f_i})$ is the anchor (forget embedding), $p_i=\zeta_\theta(x_{f_i}^+)$ is its positive (paraphrased or dropout-augmented), $n_j=\zeta_\theta(x_{r_j})$ is a negative (retain embedding), and $s(u,v)=u^\top v$.  Suppose $a_i,p_i,n_j\in\mathbb{S}^{d-1}$.  The gradient of $\ell_{ij}$ with respect to $a_i$ is
    \[
    \nabla_{a_i}\ell_{ij}
    = -2\,\sigma\!\Bigl(-\tfrac{m_{ij}}{\tau}\Bigr)\,\bigl(p_i - n_j\bigr).
    \]
    In particular, a gradient descent step $a_i^{\prime} = a_i - \eta\,\nabla_{a_i}\ell_{ij}$ (with small $\eta>0$) moves $a_i$ \emph{toward} $p_i$ and \emph{away from} $n_j$.  If $p_i\neq n_j$ and $\eta>0$, then ${a_i^{\prime}}^\top n_j \;<\; a_i^\top n_j$, so the anchor--negative cosine similarity strictly decreases.
\end{proposition}
The update direction $p_i - n_j$ splits into $+p_i$ (an alignment pull toward the positive view) and $-n_j$ (a repulsion from a retain negative): the per-sample analogues of the alignment and uniformity gradients in \citet{wang2020understanding}, restricted to one cross-distribution pair from $(\mathcal{F},\mathcal{R})$.

\textbf{Multiple negatives.} Summed over all $B_r$ negatives in Eq.~\eqref{eq:dpo-cl}, the anchor is repelled from the \emph{centroid} $\bar n_w$ of the negatives weighted by $w_{ij}=\sigma(-m_{ij}/\tau)$, which up-weights hard negatives automatically. Its similarity to $\bar n_w$ strictly decreases iff $\cos(p_i,\bar n_w)<\|\bar n_w\|$ (Prop.~\ref{prop:multineg}, Supp.~\ref{supp:proof}), which reduces to $p_i\neq n_j$ when $B_r=1$. This condition is not automatic; measured during training, it holds for $94\%$ of anchor updates ($16{,}660$ of $17{,}728$), an empirical regularity rather than a guarantee.
\begin{corollary}
\label{coro:similarity}
    \textbf{(One-step decrease of cross-similarity)}. Under the same setup as Proposition~\ref{prop:cl}, consider the expected cross-similarity (linear kernel overlap)
    \[
    C_{\mathrm{lin}}(\theta)
    =\mathbb{E}\bigl[a^\top n\bigr],
    \]
    where $a=\zeta_\theta(x_f)$ and $n=\zeta_\theta(x_r)$ for independent $x_f\in\mathcal{F}$ and $x_r\in\mathcal{R}$.  If a single gradient step on $\ell_{ij}$ updates $a_i$ to $a_i^{\prime}$ while leaving $p_i$ and all $n_j$ fixed, then with updated parameters $\theta^{\prime}$,
    \[
    C_{\mathrm{lin}}(\theta^{\prime}) \;\le\; C_{\mathrm{lin}}(\theta),
    \]
    with strict inequality if $p_i\neq n_j$ for any updated pair. Hence, the DPO-CL update strictly reduces the expected anchor–negative similarity.
\end{corollary}
\begin{proposition}
\label{prop:separate}
    \textbf{(Increase of separation under CLReg)}. Let $D$ be any separation measure between distributions satisfying the following:
    \begin{itemize}
        \item There exists a continuous cost function $c:\mathbb{R}^d\times \mathbb{R}^d\to\mathbb{R}$ such that $D(\mathcal{P},\mathcal{Q})$ is a non-decreasing function of the expected cross-cost $\mathbb{E}_{u\sim \mathcal{P},\,v\sim \mathcal{Q}}\![\,c(u,v)\,]$; that is,
        \(
        c(u,v_1)\le c(u,v_2)\;\text{implies}\;D(P,\delta_{v_1})\le D(P,\delta_{v_2}),
        \)
        where $\delta_v$ denotes the Dirac measure at $v$.
        \item The cost $c$ is strictly increasing with respect to the anchor–negative similarity: if $s(u_1,v)<s(u_2,v)$ then $c(u_1,v) > c(u_2,v)$.
    \end{itemize}
    Then a gradient descent step on $\mathcal{L}_{\text{CL}}^{\text{dpo}}$ reduces $\mathbb{E}[s(a,n)]$ and thereby \emph{increases} $D(\mathcal{P}_{\mathcal{F}}^\theta,\mathcal{P}_{\mathcal{R}}^\theta)$:
    \[
    D\bigl(\mathcal{P}_{\mathcal{F}}^{\theta^{\prime}},\,\mathcal{P}_{\mathcal{R}}^{\theta^{\prime}}\bigr)
    \;\ge\;
    D\bigl(\mathcal{P}_{\mathcal{F}}^{\theta},\,\mathcal{P}_{\mathcal{R}}^{\theta}\bigr),
    \]
    with strict increase when at least one updated anchor has $p_i\neq n_j$.
\end{proposition}
See Supp.~\ref{supp:proof} for all proofs. The cross-similarity $C_{\mathrm{lin}}(\theta)=\mathbb{E}[a^\top n]$ is the inter-class kernel overlap whose contraction underlies the generalization bound of \citet{arora2019theoretical}, so Prop.~\ref{prop:separate} shows that one CLReg step moves the same quantity their analysis tracks, in the same direction, and explains why this should reduce embedding-level entanglement between $\mathcal{C}_{\mathcal{F}}^{\theta}$ and $\mathcal{C}_{\mathcal{R}}^{\theta}$~\cite{zhao2024makes}. The empirical reductions in Tab.~\ref{tab:tofu-feat}--\ref{tab:muse-feat} and the UMAP visualizations in Fig.~\ref{fig:umap}, not the propositions, are where the claim is ultimately established.
\section{Experiment}
\label{sec:experiments}
\vspace{-5pt}
\begin{table*}[ht!]
\centering
\small
\caption{TOFU unlearning experiments for Llama-3-8B and 3B.}
\label{tab:tofu-all}
\begin{subtable}{\linewidth}
\centering
\caption{TOFU unlearning experiment results for Llama-3-8B. For the last four columns, bold indicates the best in-column, and green shades indicate improvement. CLReg consistently improves overall unlearning score and model utility. SimNPO+CL achieves the best performance, and GradDiff+CL achieves the largest improvement. In most cases, CLReg brings privacy leak closer to $0$, achieving better balance.}
\label{tab:tofu-8b}
\resizebox{\linewidth}{!}{%
\begin{tabular}{l|cccc|cccc}
\toprule
\shortstack[c]{\textbf{TOFU}\\\textbf{Llama-3-8B}} & \shortstack[c]{Extraction\\Strength$\downarrow$} & \shortstack[c]{Forget QA\\Prob$\downarrow$} & \shortstack[c]{Forget QA\\ROUGE$\downarrow$} & \shortstack[c]{Forget\\Quality$\uparrow$} & \shortstack[c]{Forget Score$\uparrow$} & \shortstack[c]{Model Utility$\uparrow$} & \shortstack[c]{\textbf{Unlearning} \textbf{Score}$\uparrow$} & \shortstack[c]{Privacy Leak$\rightarrow$0} \\
\midrule
GradDiff     &0.06646	&0.00498	&0.28313	&2.63E-10	&0.86826	&0.53938	&0.66540	&51.03416\\
UNDIAL       &0.06123	&0.20657	&0.30030	&4.64E-12	&0.80476	&0.62101	&0.70104	&-78.91192\\
NPO          &0.12513	&0.30622	&0.39238	&1.37E-07	&0.83696	&0.66680	&0.74225	&-69.79255\\
SimNPO       &0.12980	&0.10784	&0.39639	&4.46E-06	&0.91079	&0.67261	&0.77378	&-48.65591\\
PDU          &0.06233	&0.04992	&0.19536	&4.36E-09	&0.89069	&0.57892	&0.70173	&45.72288\\
RMU  & 0.03361 & 0.00278 & 0.03627 & 5.7E-15 & 0.75790 & 0.65646 & 0.70354 & 63.85000\\
\midrule
GradDiff+CL     &0.14964	&0.32726	&0.39953	&9.91E-11	&0.77768	&\cellcolor{green!15}0.67359	&\cellcolor{green!15}0.72191	&\cellcolor{green!15}-48.86873\\
UNDIAL+CL       &0.06201	&0.22942	&0.31487	&9.34E-13	&0.78446	&\cellcolor{green!15}0.66806	&\cellcolor{green!15}0.72159	&\cellcolor{green!15}-76.47375\\
NPO+CL          &0.11762	&0.28935	&0.38650	&1.30E-05	&\cellcolor{green!15}0.87003	&\cellcolor{green!15}0.68487	&\cellcolor{green!15}0.76643	&\cellcolor{green!15}-56.01587\\
SimNPO+CL       &0.04593	&0.01815	&0.11235	&0.00229	&\cellcolor{green!15}\textbf{0.97182}	&\cellcolor{green!15}\textbf{0.69815}	&\cellcolor{green!15}\textbf{0.81256}	&55.03337\\
PDU+CL          &0.06375	&0.07160	&0.22932	&6.78E-07	&\cellcolor{green!15}0.92574	&\cellcolor{green!15}0.58539	&\cellcolor{green!15}0.71724	&\cellcolor{green!15}\textbf{29.14540}\\
\bottomrule
\end{tabular}}
\end{subtable}
\vspace{-5pt}
\begin{subtable}{\linewidth}
\centering
\caption{TOFU unlearning experiment results for Llama-3-3B. For the last four columns, bold indicates the best in-column, and green shades indicate improvement. CLReg consistently improves overall unlearning scores and model utility. SimNPO+CL achieves the best performance. In most cases, CLReg brings privacy leak closer to $0$. Impactfully, SimNPO+CL reduces the original absolute privacy leak value $61.994$ to $3.725$.}
\label{tab:tofu-3b}
\resizebox{\linewidth}{!}{%
\begin{tabular}{l|cccc|cccc}
\toprule
\shortstack[c]{\textbf{TOFU}\\\textbf{Llama-3-3B}} & \shortstack[c]{Extraction\\Strength$\downarrow$} & \shortstack[c]{Forget QA\\Prob$\downarrow$} & \shortstack[c]{Forget QA\\ROUGE$\downarrow$} & \shortstack[c]{Forget\\Quality$\uparrow$} & \shortstack[c]{Forget Score$\uparrow$} & \shortstack[c]{Model Utility$\uparrow$} & \shortstack[c]{\textbf{Unlearning} \textbf{Score}$\uparrow$} & \shortstack[c]{Privacy Leak$\rightarrow$0} \\
\midrule
GradDiff     &0.08236	&0.03150	&0.34287	&7.83E-12	&0.80783	&0.59254	&0.68364	&-2.41669\\
UNDIAL       &0.06051	&0.18277	&0.23129	&1.73E-15	&0.67538	&0.56827	&0.61721	&-86.86811\\
NPO          &0.12257	&0.35523	&0.41255	&2.77E-09	&0.77438	&0.61585	&0.68607	&-75.13616\\
SimNPO       &0.11093	&0.11723	&0.38107	&6.83E-09	&0.85847	&0.61040	&0.71348	&-61.99405\\
PDU          &0.07656	&0.17840	&0.27305	&6.39E-06	&0.90204	&0.50306	&0.64590	&-42.08622\\
RMU  & 0.03251 & 0.00013 & 0.10726 & 5.42E-13 & 0.77818 & 0.62875 & 0.69553 & 63.14000\\
\midrule
GradDiff+CL     &0.06042	&0.10162	&0.31587	&0.52341	&\cellcolor{green!15}\textbf{0.98902}	&\cellcolor{green!15}0.63042	&\cellcolor{green!15}0.77001	&6.36311\\
UNDIAL+CL       &0.06137	&0.20976	&0.26059	&5.14E-16	&0.65019	&\cellcolor{green!15}0.61893	&\cellcolor{green!15}0.63417	&\cellcolor{green!15}-84.90512\\
NPO+CL          &0.09906	&0.28907	&0.38257	&6.78E-07	&\cellcolor{green!15}0.84701	&\cellcolor{green!15}0.63707	&\cellcolor{green!15}0.72719	&\cellcolor{green!15}-64.04083\\
SimNPO+CL       &0.06353	&0.10866	&0.32609	&0.00383	&\cellcolor{green!15}0.96209	&\cellcolor{green!15}\textbf{0.66531}	&\cellcolor{green!15}\textbf{0.78664}	&\cellcolor{green!15}\textbf{-3.72569}\\
PDU+CL          &0.06178	&0.08528	&0.26582	&2.57E-05	&\cellcolor{green!15}0.93845	&\cellcolor{green!15}0.52995	&\cellcolor{green!15}0.67738	&\cellcolor{green!15}18.51810\\
\bottomrule
\end{tabular}}
\end{subtable}
\vspace{-10pt}
\end{table*}
\subsection{Unlearning Setup}
\vspace{-5pt}
We conduct unlearning experiments on TOFU~\cite{maini2024tofu} and MUSE~\cite{shi2024muse} benchmarks (overview and concrete examples in Supp.~\ref{supp:data}), where on TOFU we unlearn LLMs of different sizes (Llama-3.1-8B and Llama-3.2-3B~\cite{grattafiori2024llama}), and on MUSE we experiment unlearning both Books and News datasets with Llama-2-7B~\cite{touvron2023llama}. Given the finetuned model $\mathbf{W}_{\text{FT}}$, we unlearn it for $10$ epochs with lr$=10^{-5}$ to obtain $\mathbf{W}_{\text{UL}}$. We unlearn with GradDiff, NPO, SimNPO, UnDIAL, PDU~\cite{zhang2024negative,fan2024simplicity,dong2024undial,entesari2025constrained}. We first tune method-specific hyper-parameters and $\gamma$ for each method for optimal performance as baselines, then tune CLReg-specific parameters when being applied: $\tau\times[\text{symmetric, non-symmetric}]\times[\mathcal{L}_{\text{CL}}^{\text{dpo}},\mathcal{L}_{\text{CL}}^{\text{info}}]$. We fix $\alpha,\lambda=1$ to ease hyper-parameter search. Paraphrases are generated once offline with Llama-3.1-8B-Instruct (five per question, question only), and we train with an effective batch size of $128$ (TOFU) / $256$ (MUSE) and $32$ negatives per anchor. See Supp.~\ref{supp:exp-settings} and~\ref{supp:impl} for detailed settings.\footnote{The analyses in Sec.~\ref{sec:robust} and Supp.~\ref{supp:new} re-run every arm, baselines included, from the publicly released benchmark checkpoints, so absolute scores differ slightly from Tabs.~\ref{tab:tofu-all}--\ref{tab:muse-all} (e.g., SimNPO on Llama-3-3B: $0.688$ vs.\ $0.713$), while every comparison within those analyses is matched.}

\subsection{Evaluation}
\vspace{-5pt}
In addition to adopting evaluation metrics from TOFU and MUSE, we propose \textbf{Forget Score}$\uparrow$ that maps each forget metric of $\mathbf{W}_{\text{UL}}$ as a progress measure from $\mathbf{W}_{\text{FT}}$ to $\mathbf{W}_{\text{RT}}$: the more unlearned the $\mathbf{W}_{\text{UL}}$ is, the more similar performance it is expected to share with $\mathbf{W}_{\text{RT}}$ on $\mathcal{F}$. For each forget metric $m$, first convert it to a progress measure:
\begin{equation}
    \mathrm{Prog}(m, \mathcal{F}) = \frac{|m(f_{\text{UL}},\mathcal{F}) - m(f_{\text{FT}},\mathcal{F})|}{|m(f_{\text{RT}},\mathcal{F}) - m(f_{\text{FT}},\mathcal{F})|},
\end{equation}
which will be clipped at $1$ when outperforming $\mathbf{W}_{\text{RT}}$. Note that the \texttt{ForgetQuality}~\cite{maini2024tofu} spans many orders of magnitude, and is usually close to zero, we take $\log(\cdot)$ on it to better address the small differences. Given $K$ evaluation metrics to measure different aspects of forgetting, we compute an overall \texttt{ForgetScore} analogous to \texttt{ModelUtility} as the harmonic mean:
\begin{equation}
    \mathrm{ForgetScore} = K\left(\sum_{k=1}^{K}\frac{1}{\mathrm{Prog}(m_k, \mathcal{F})}\right)^{-1}.
\end{equation}
Likewise, we can measure an overall \textbf{Unlearning Score}$\uparrow$ as the harmonic mean of \texttt{ForgetQuality} and \texttt{ModelUtility} (or \texttt{RetainKnowmemROUGE} on MUSE), emphasizing on balancing forgetting and retaining. Despite many of the metrics are privacy/leakage-aware already (e.g., \texttt{ForgetQuality})~\cite{dorna2025openunlearning}, we also report \texttt{PrivLeak}$\rightarrow0$ dedicated to privacy leakage~\cite{shi2024muse}, where positive suggests over-unlearning and negative suggests under-unlearning and is encouraged to approach zero when $\mathbf{W}_{\text{UL}}$ is well-balanced. See Supp.~\ref{supp:metrics-overview} for an overview of each evaluation metric.

\subsection{Shaping Representation Improves Unlearning}
\vspace{-5pt}
\begin{table*}[tb]
\centering
\small
\caption{MUSE unlearning experiments across datasets (Books and News).}
\label{tab:muse-all}
\begin{subtable}{\linewidth}
\centering
\caption{MUSE-Books Unlearning experiment results. For the last four columns, bold indicates the best in-column, and green shades indicate improvement. CLReg consistently improves overall unlearning scores and forget scores, and can outperform retrained models in forgetting. SimNPO+CL achieves the best performance, and GradDiff+CL achieves the largest improvement. Similar to TOFU results, CLReg does not degrade privacy leak or undermines unlearning balance, and bring it closer to $0$ for many cases.}
\label{tab:muse-books}
\resizebox{\linewidth}{!}{%
\begin{tabular}{l|cccc|cccc}
\toprule
\shortstack[c]{\textbf{MUSE-Books}\\\textbf{Llama-2-7B}} & \shortstack[c]{Exact\\Memoriation$\downarrow$} & \shortstack[c]{Extraction\\Strength$\downarrow$} & \shortstack[c]{Forget knowmem\\ROUGE$\downarrow$} & \shortstack[c]{Forget verbmem\\ROUGE$\downarrow$} & \shortstack[c]{Forget\\Score$\uparrow$} & \shortstack[c]{Retain knowmem\\ROUGE$\uparrow$} & \shortstack[c]{\textbf{Unlearning} \textbf{Score}$\uparrow$} & \shortstack[c]{Privacy\\Leak$\rightarrow$0} \\
\midrule
GradDiff    & 0.93627	&0.15484	&0.32808	&0.35574	&0.31345	&0.59876	&0.41148	&-60.46598 \\
UNDIAL      & 0.94730	&0.18230	&0.38053	&0.42416	&0.25306	&\textbf{0.63058}	&0.36117	&-53.93861 \\
NPO         & 0.92421	&0.13040	&0.20335	&0.30689	&0.36774	&0.52260	&0.43170	&-51.25740 \\
SimNPO      & 0.20746	&0.00913	&0.33610	&0.00316	&0.94209	&0.60877	&0.73961	&-30.65828 \\
PDU         & 0.00968	&0.00794	&0.22401	&8.85E-05	&\textbf{1.0}	    &0.44479	&0.61572	&-35.96524 \\
RMU        & 0.62484 & 0.02341 & 0.3716 & 0.15719 & 0.77996 & 0.61749 & 0.68928 & -55.42000\\
\midrule
GradDiff+CL    &0.18952	&0.00865	&0.18392	&0.00540	&\cellcolor{green!15}\textbf{1.0}	    &0.58451	&\cellcolor{green!15}0.73778	&\cellcolor{green!15}-57.32249\\
UNDIAL+CL      &0.90897	&0.09333	&0.33523	&0.27550	&\cellcolor{green!15}0.41405	&0.57313	&\cellcolor{green!15}0.48077	&-57.56287\\
NPO+CL         &0.91302	&0.11452	&0.21475	&0.26978	&\cellcolor{green!15}0.41021	&\cellcolor{green!15}0.53345	&\cellcolor{green!15}0.46378	&-52.55178\\
SimNPO+CL      &0.01206	&0.00794	&0.27470	&0.0	    &\cellcolor{green!15}\textbf{1.0}	    &0.59918	&\cellcolor{green!15}\textbf{0.74936}	&\cellcolor{green!15}\textbf{-21.11686}\\
PDU+CL         &0.00944	&0.00794	&0.23983	&3.37E-04	&\cellcolor{green!15}\textbf{1.0}	    &\cellcolor{green!15}0.47345	&\cellcolor{green!15}0.64264	&\cellcolor{green!15}-34.26405\\
\bottomrule
\end{tabular}}
\end{subtable}
\vspace{-5pt}
\begin{subtable}{\linewidth}
\centering
\caption{MUSE-News Unlearning experiment results. For the last four columns, bold indicates the best in-column, and green shades indicate improvement. CLReg consistently improves overall unlearning scores. SimNPO+CL achieves the best performance with slight over-unlearning despite improving $|\texttt{PrivLeak}|$.}
\label{tab:muse-news}
\resizebox{\linewidth}{!}{%
\begin{tabular}{l|cccc|cccc}
\toprule
\shortstack[c]{\textbf{MUSE-News}\\\textbf{Llama-2-7B}} & \shortstack[c]{Exact\\Memoriation$\downarrow$} & \shortstack[c]{Extraction\\Strength$\downarrow$} & \shortstack[c]{Forget knowmem\\ROUGE$\downarrow$} & \shortstack[c]{Forget verbmem\\ROUGE$\downarrow$} & \shortstack[c]{Forget\\Score$\uparrow$} & \shortstack[c]{Retain knowmem\\ROUGE$\uparrow$} & \shortstack[c]{\textbf{Unlearning} \textbf{Score}$\uparrow$} & \shortstack[c]{Privacy\\Leak$\rightarrow$0} \\
\midrule
GradDiff     &0.83810	&0.06833	&0.53086	&0.24138	&0.48643	&0.46619	&0.47609	&\textbf{-87.84635}\\
UNDIAL       &0.75079	&0.03571	&0.28436	&0.20273	&0.84017	&0.36324	&0.50720	&-99.49622\\
NPO          &0.78690	&0.04421	&0.50151	&0.20549	&0.62377	&0.45399	&0.52550	&-92.69521\\
SimNPO       &0.84802	&0.07198	&0.54494	&0.24493	&0.44588	&\textbf{0.47777}	&0.46128	&-90.55416\\
PDU          &0.77754	&0.02984	&0.41276	&0.20177	&0.74393	&0.35294	&0.47875	&-99.72712\\
RMU       & 0.81357 & 0.04881 & 0.51091 & 0.24973 &0.55910 &0.46187 &0.50586 &-99.33000\\
\midrule
GradDiff+CL     &0.34405	&0.01198	&0.42668	&0.03450	&\cellcolor{green!15}\textbf{0.89807}	&0.41004	&\cellcolor{green!15}0.56302	&99.68514\\
UNDIAL+CL       &0.75532	&0.03278	&0.27290	&0.21257	&0.83035	&\cellcolor{green!15}0.36937	&\cellcolor{green!15}0.51130	&\cellcolor{green!15}-99.47523\\
NPO+CL          &0.78722	&0.04500	&0.49930	&0.20703	&\cellcolor{green!15}0.62591	&\cellcolor{green!15}0.46569	&\cellcolor{green!15}0.53404	&\cellcolor{green!15}-92.54828\\
SimNPO+CL       &0.38992	&0.01413	&0.44140	&0.04043	&\cellcolor{green!15}0.87730	&0.42403	&\cellcolor{green!15}\textbf{0.57173}	&\cellcolor{green!15}88.65659\\
PDU+CL          &0.78397	&0.02786	&0.39836	&0.21609	&0.73958	&\cellcolor{green!15}0.37476	&\cellcolor{green!15}0.49745	&-99.72712\\
\bottomrule
\end{tabular}}
\end{subtable}
\vspace{-10pt}
\end{table*}
Tab.~\ref{tab:tofu-all} and Tab.~\ref{tab:muse-all} report detailed results on TOFU and MUSE. CLReg consistently improves \texttt{UnlearningScore} across LLM sizes and datasets, lifting \texttt{ForgetScore} in most cases while \emph{also} improving \texttt{ModelUtility} for all methods on TOFU. CLReg can yield mild over-unlearning (positive \texttt{PrivLeak}) in a few cases but does not noticeably increase $|\texttt{PrivLeak}|$, and in most settings pulls \texttt{PrivLeak} closer to $0$ -- the desired balance per~\cite{shi2024muse}. Overall, SimNPO+CL achieves the best performance across all settings, and CLReg most steadily strengthens SimNPO and NPO; we hypothesize their preference-learning objectives share more optimization goals with CLReg than the other methods. The gain is not tied to a particular contrastive design: a $2{\times}2$ ablation on contrastive-loss form ($\mathcal{L}_{\text{CL}}^{\text{dpo}}$ vs.\ $\mathcal{L}_{\text{CL}}^{\text{info}}$) and symmetry shows $14/16$ configurations beat the CLReg-free baseline (Supp.~\ref{supp:cl-ablation}). Nor does it need many or carefully chosen negatives: one negative gives $0.803$ versus $0.807$ for $32$, and hard, random, or easy negatives span $0.012$ (Supp.~\ref{supp:negatives}); CLReg also helps on Llama-3.2-1B (Supp.~\ref{supp:size}). The largest regression, the \texttt{PrivLeak} of GradDiff+CL on MUSE-News, reflects the fixed $\lambda=1$ rather than the method: lowering $\lambda$ brings it through zero (Supp.~\ref{supp:lambda}). CLReg adds only one extra forward pass per iteration and no additional trainable parameters, yielding a $34$--$49\%$ wall-clock overhead on TOFU (Supp.~\ref{supp:compute}).

\textbf{Comparison with RMU.} Unlike the base methods above, RMU~\cite{li2024wmdp} operates on hidden representations by misdirecting forget-layer activations toward a \emph{random} direction. We add RMU as an additional baseline row in Tables~\ref{tab:tofu-8b}--\ref{tab:tofu-3b} and \ref{tab:muse-books}--\ref{tab:muse-news}. On TOFU, RMU yields \texttt{UnlearningScore} $0.70$ (8B) / $0.70$ (3B), while SimNPO+CL reaches $0.81$ / $0.79$; on MUSE-Books, RMU reaches $0.69$ versus SimNPO+CL $0.75$; on MUSE-News, RMU reaches $0.51$ versus SimNPO+CL $0.57$. Most CLReg-augmented methods therefore outperform RMU, which is consistent with our view that \emph{structured} contrastive separation (forget away from retain) is more effective than \emph{random} misdirection for retain-side preservation. We do not report RMU+CL: the two operate on representations with potentially conflicting objectives (random misdirection vs.\ structured contrastive separation), and we leave a careful study to future work.

\textbf{WMDP.} We further evaluate CLReg on WMDP-cyber~\cite{li2024wmdp} with Zephyr-7B-$\beta$, where a true retrained reference is infeasible (we use the random-guess MCQ accuracy $0.25$ as a conservative lower-bound proxy for $m(\mathbf{W}_{\text{RT}},\mathcal{F})$). CLReg improves \texttt{UnlearningScore} for \emph{all five} base methods; SimNPO+CL reaches $0.674$ and essentially matches RMU's published $0.677$, and NPO+CL lifts NPO from $0.551$ to $0.584$ once $\alpha$ is rebalanced. Full setup, table, and per-method discussion are in Supp.~\ref{supp:wmdp}.

\subsection{Disentangled Representation for Easier Unlearning}
\vspace{-5pt}
\begin{figure*}[tb]
\centering
\begin{subfigure}[tb]{0.25\linewidth}
\includegraphics[width=\linewidth]{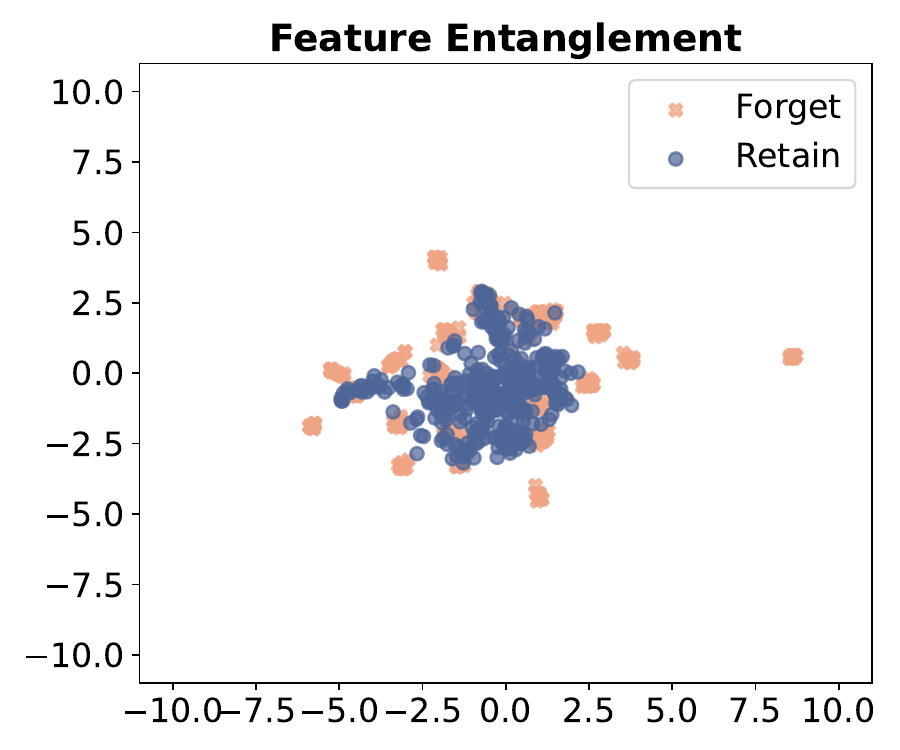}
\subcaption*{NPO, Llama-3-8B}
\end{subfigure}\hfill
\begin{subfigure}[tb]{0.25\linewidth}
\includegraphics[width=\linewidth]{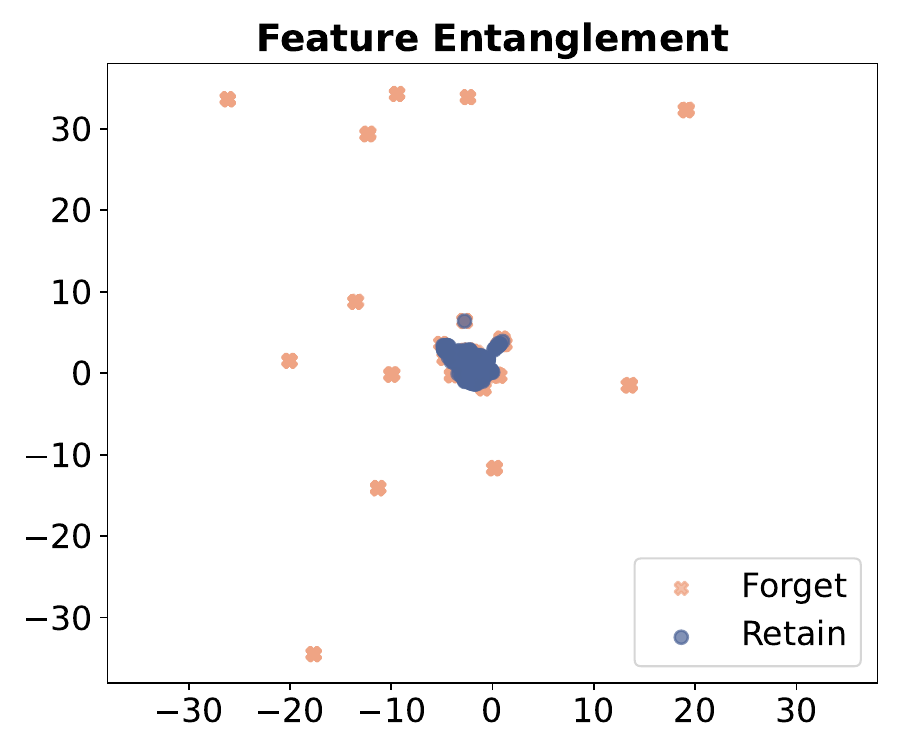}
\subcaption*{NPO+CL, Llama-3-8B}
\end{subfigure}\hfill
\begin{subfigure}[tb]{0.25\linewidth}
\includegraphics[width=\linewidth]{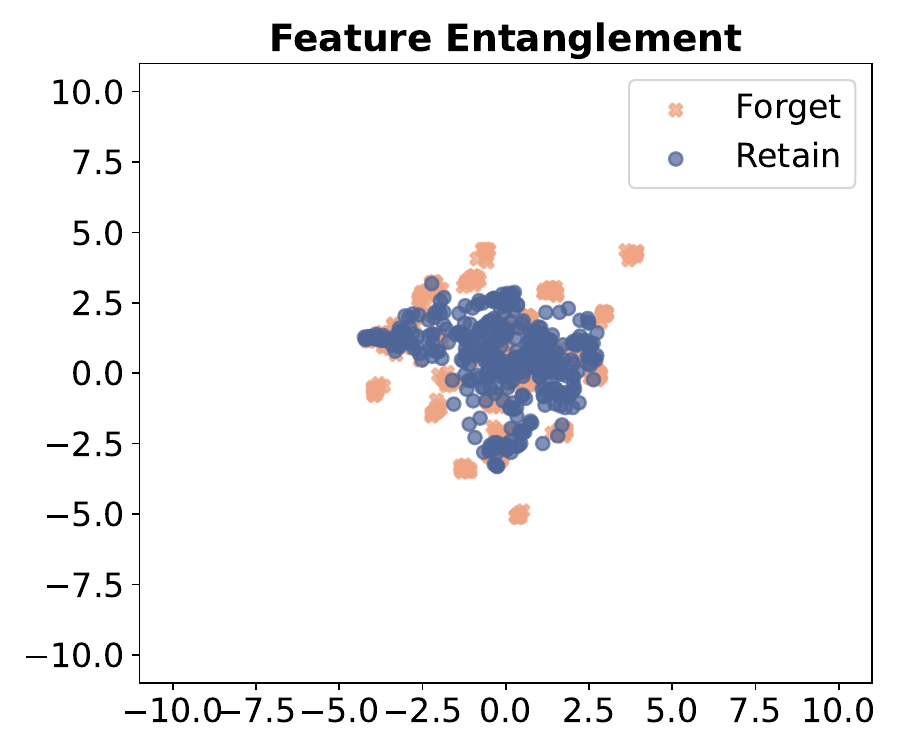}
\subcaption*{SimNPO, Llama-3-8B}
\end{subfigure}\hfill
\begin{subfigure}[tb]{0.25\linewidth}
\includegraphics[width=\linewidth]{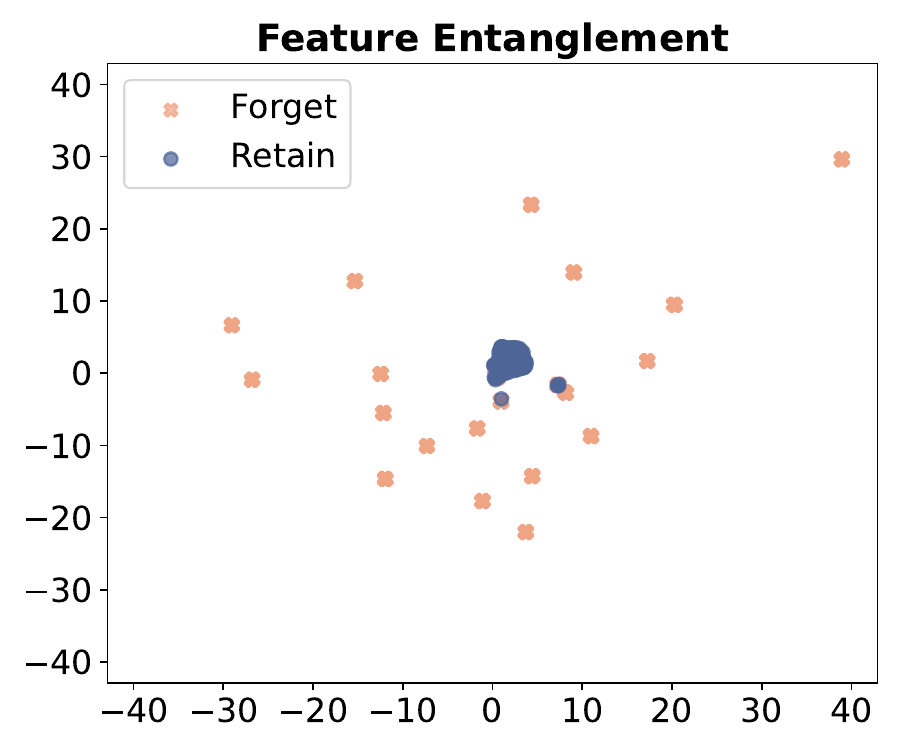}
\subcaption*{SimNPO+CL, Llama-3-8B}
\end{subfigure}
\begin{subfigure}[tb]{0.25\linewidth}
\includegraphics[width=\linewidth]{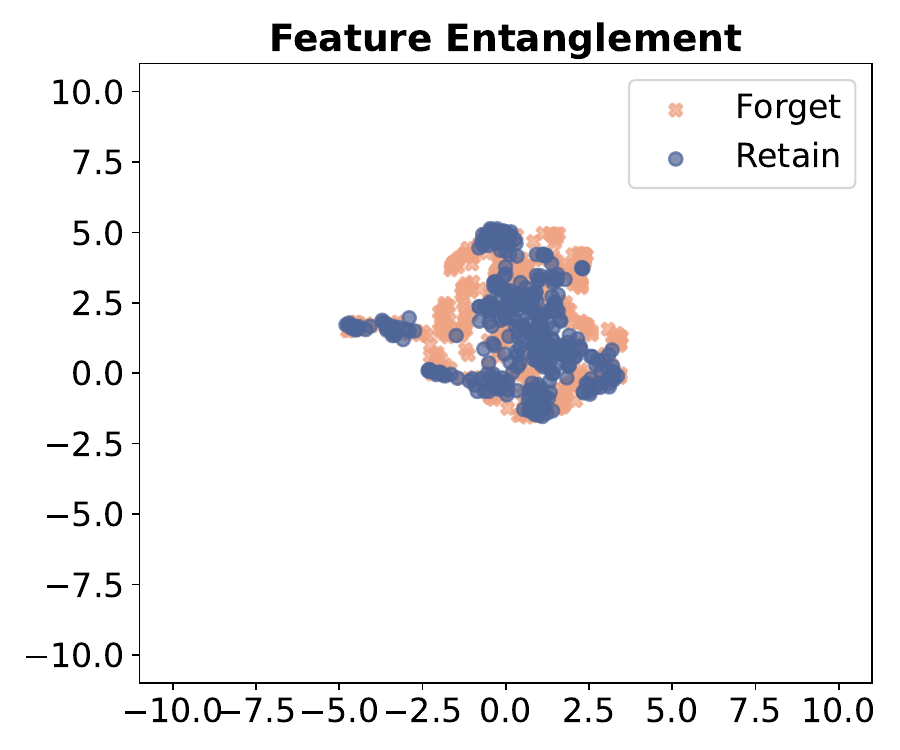}
\subcaption*{NPO, Llama-3-3B}
\end{subfigure}\hfill
\begin{subfigure}[tb]{0.25\linewidth}
\includegraphics[width=\linewidth]{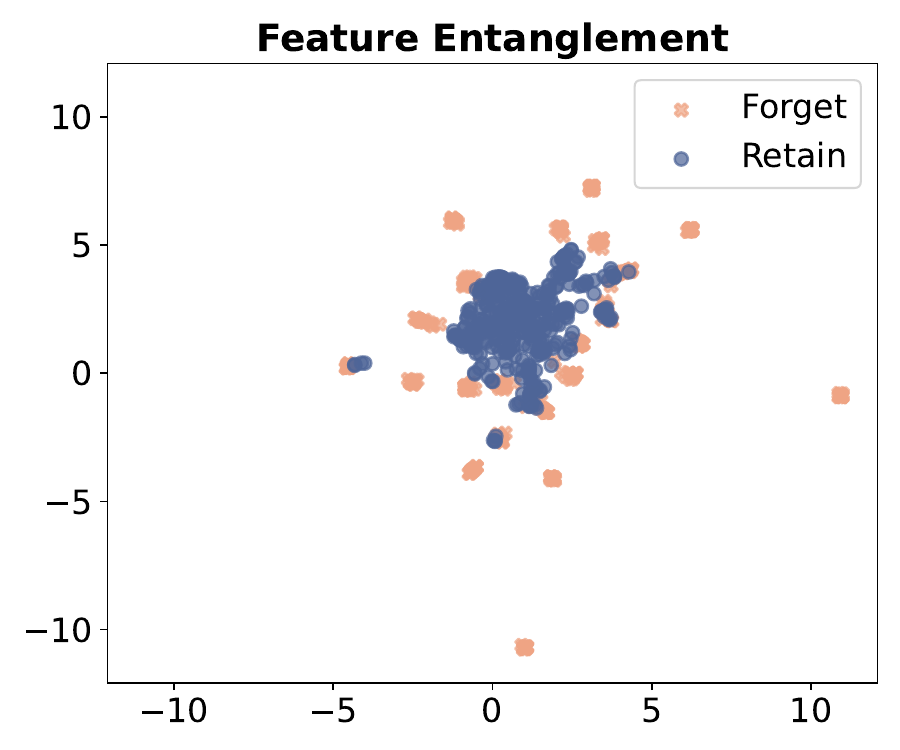}
\subcaption*{NPO+CL, Llama-3-3B}
\end{subfigure}\hfill
\begin{subfigure}[tb]{0.25\linewidth}
\includegraphics[width=\linewidth]{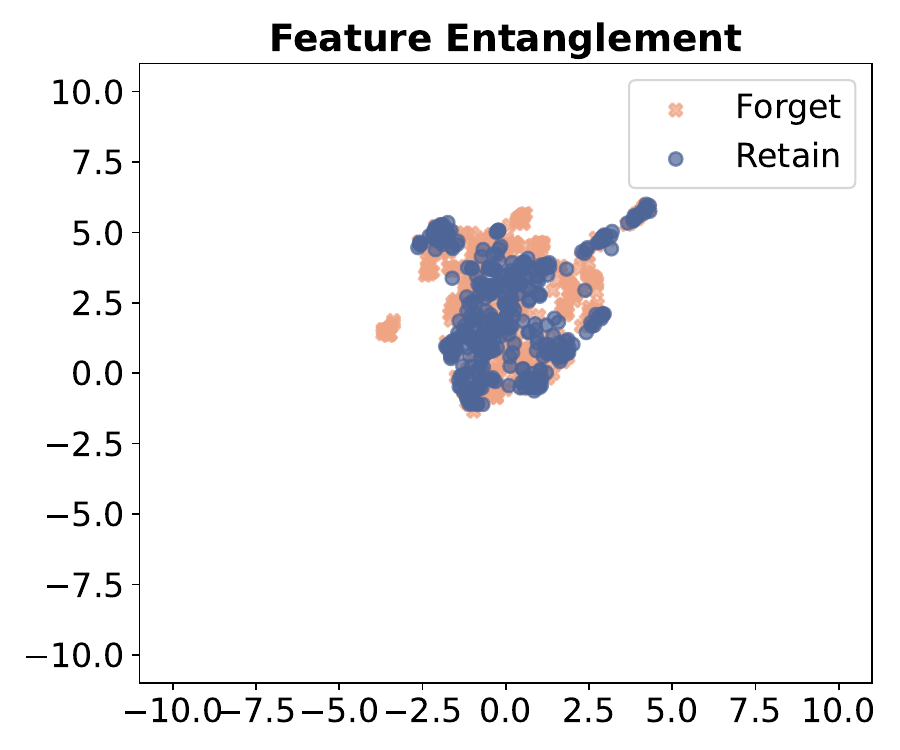}
\subcaption*{SimNPO, Llama-3-3B}
\end{subfigure}\hfill
\begin{subfigure}[tb]{0.25\linewidth}
\includegraphics[width=\linewidth]{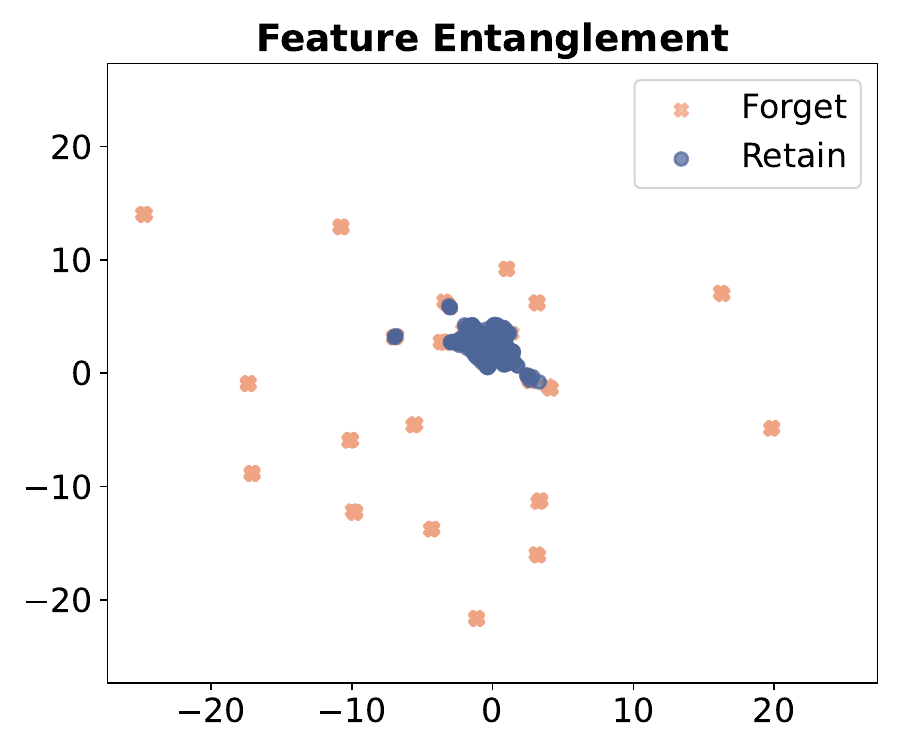}
\subcaption*{SimNPO+CL, Llama-3-3B}
\end{subfigure}
\caption{UMAP visualizations of NPO and SimNPO unlearning on TOFU benchmark, compared with CLReg variants. We observe that CLReg can effectively identify and separate forget features by pushing them away, while still maintaining the original scale and distributions of the retain features. \textbf{Please compare axis scales}: without CLReg, both forget and retain features lie within roughly $[-10,10]$. With CLReg, the retain cluster remains in $[-10,10]$ with essentially the same shape and density, while forget features are pushed to a much larger range (e.g., $[-30,40]$ for SimNPO-8B).}
\vspace{-10pt}
\label{fig:umap}
\end{figure*}

\begin{table}[htb!]
\vspace{-10pt}
\centering
\small
\begin{minipage}[t]{0.49\linewidth}
    \caption{TOFU entanglement evaluation results. CLReg consistently reduces entanglement on all three metrics for NPO and SimNPO; the largest improvement is SimNPO+CL on Llama-3-8B ($20.24\to5.9$). $^\dagger$Re-run from the public checkpoint (Sec.~\ref{sec:robust}), correcting a duplicated entry in the submitted version.}
    \label{tab:tofu-feat}
    \begin{minipage}{\linewidth}
    \centering
    \resizebox{\linewidth}{!}{%
    \begin{tabular}{l|cc|cc}
    \toprule
    \textbf{Llama-3-8B} & NPO & +CL & SimNPO & +CL \\
    \midrule
    Entanglement$\downarrow$    & 22.30576 &\cellcolor{green!15}17.27656 & 20.24046 & \cellcolor{green!15}5.90206 \\
    MK-MMD$\uparrow$          & 0.01424 &\cellcolor{green!15}0.02138 & 0.01537 & \cellcolor{green!15}0.07327 \\
    2-Wasserstein$\uparrow$   & 0.10355 &\cellcolor{green!15}0.15881 & 0.10001 & \cellcolor{green!15}0.31250 \\
    \bottomrule
    \end{tabular}}
    \end{minipage}
    \begin{minipage}{\linewidth}
    \centering
    \resizebox{\linewidth}{!}{%
    \begin{tabular}{l|cc|cc}
    \toprule
    \textbf{Llama-3-3B} & NPO & +CL & SimNPO$^\dagger$ & +CL$^\dagger$ \\
    \midrule
    Entanglement$\downarrow$    & 25.07525 &\cellcolor{green!15}18.02604 &17.19105 &\cellcolor{green!15}5.62938 \\
    MK-MMD$\uparrow$          & 0.01213 &\cellcolor{green!15}0.01820 &0.01833 &\cellcolor{green!15}0.06930\\
    2-Wasserstein$\uparrow$   & 0.06348 &\cellcolor{green!15}0.10085 &0.06768 &\cellcolor{green!15}0.23700\\
    \bottomrule
    \end{tabular}}
    \end{minipage}
\end{minipage}\hfill
\begin{minipage}[t]{0.49\linewidth}
    \caption{MUSE entanglement evaluation results. CLReg consistently reduces entanglement on all three metrics for NPO and SimNPO; the largest improvement is SimNPO+CL on MUSE-News ($92.18\to11.09$).}
    \label{tab:muse-feat}
    \begin{minipage}{0.98\linewidth}
    \centering
    \resizebox{\linewidth}{!}{%
    \begin{tabular}{l|cc|cc}
    \toprule
    \textbf{MUSE-Books} & NPO & +CL & SimNPO & +CL \\
    \midrule
    Entanglement$\downarrow$    &0.36104 &\cellcolor{green!15}0.28346 &0.03898 &\cellcolor{green!15}0.02058 \\
    MK-MMD$\uparrow$          &0.47111 &\cellcolor{green!15}0.57001 &1.21513 &\cellcolor{green!15}1.27422 \\
    2-Wasserstein$\uparrow$   &0.17936 &\cellcolor{green!15}0.23721 &0.62096 &\cellcolor{green!15}0.90730 \\
    \bottomrule
    \end{tabular}}
    \end{minipage}
    \begin{minipage}{0.98\linewidth}
    \centering
    \resizebox{\linewidth}{!}{%
    \begin{tabular}{l|cc|cc}
    \toprule
    \textbf{MUSE-News} & NPO & +CL & SimNPO & +CL \\
    \midrule
    Entanglement$\downarrow$    &59.07131 &\cellcolor{green!15}58.66790 &92.17580 &\cellcolor{green!15}11.08532 \\
    MK-MMD$\uparrow$          &0.00515 &\cellcolor{green!15}0.00520 &0.00392 &\cellcolor{green!15}0.01775 \\
    2-Wasserstein$\uparrow$   &0.05856 &\cellcolor{green!15}0.05899 &0.04923 &\cellcolor{green!15}0.11468 \\
    \bottomrule
    \end{tabular}}
    \end{minipage}
\end{minipage}
\end{table}

We next examine the feature layer where CLReg is applied. Guided by the one-step analysis of Sec.~\ref{sec:motivation}, and since prior work finds that unlearning becomes harder as forget--retain entanglement grows~\cite{zhao2024makes,tang2025sharpness}, we provide quantitative and qualitative analysis of the feature space, comparing NPO, SimNPO with NPO+CL and SimNPO+CL to verify our claims. We implement variance-based entanglement from \citet{goldblum2020unraveling,zhao2024makes}: $\mathrm{E}=(\frac{1}{|\mathcal{R}|} \sum_{i \in \mathcal{R}}(\boldsymbol{\phi}_i-\boldsymbol{\mu}_{\mathcal{R}})^2+\frac{1}{|\mathcal{F}|} \sum_{j \in \mathcal{F}}(\boldsymbol{\phi}_j-\boldsymbol{\mu}_{\mathcal{F}})^2)/((\boldsymbol{\mu}_{\mathcal{R}}-\boldsymbol{\mu})^2+(\boldsymbol{\mu}_{\mathcal{F}}-\boldsymbol{\mu})^2)$, where $\boldsymbol{\phi}_i,\boldsymbol{\phi}_j$ denote sample embedding, $\boldsymbol{\mu}_{\mathcal{R}},\boldsymbol{\mu}_{\mathcal{F}}$ denote mean embedding of $\mathcal{R},\mathcal{F}$, and $\boldsymbol{\mu}$ denotes mean embedding over $\mathcal{R}\cup\mathcal{F}$. We also implement multi-kernel Maximum Mean Discrepancy (MMD) and 2-Wasserstein distance $W_2$ to comprehensively evaluate the feature separation after unlearning~\cite{JMLR:v13:gretton12a,tang2025sharpness}.

CLReg explicitly identifies and pushes away forget features, resulting in reduced entanglement. In Tab.~\ref{tab:tofu-feat} and Tab.~\ref{tab:muse-feat}, we observe that the entanglement between retain and forget features is consistently lowered, with more noticeable changes on TOFU. But does the shifted representation space alter the distributions of retained knowledge? We further visualize the feature space using UMAP~\cite{mcinnes2018umap} in Fig.~\ref{fig:umap}. Compared to the forget features, which are pushed away, CLReg does not move retain features much: while NPO and SimNPO keep both feature sets within a similar compact range, CLReg keeps the retain features at their original scale and pushes the forget features far outside it (note the axis scales in Fig.~\ref{fig:umap}). Even in the most challenging case (NPO, Llama-3-3B), where features remain more entangled after NPO unlearning, CLReg still separates the forget features from the retain cluster. The visualizations effectively demonstrate how CLReg can identify and push away forget features while keeping retained knowledge intact. The clear separation also inspires future work to clip the ``outlier'' forget features for a faithful, complete unlearning.

\subsection{Separation Facilitates and Stabilizes Unlearning}
\label{sec:robust}
\vspace{-5pt}
\begin{table}[t]
\centering
\small
\caption{Separation facilitates and stabilizes unlearning (TOFU, Llama-3.2-3B). (a) Stage~2 is an \emph{unmodified} base method; the 15-epoch run matches the warm start's total compute. (b) Normalized recovery AUC under a relearning attack; SimNPO entries are mean\,$\pm$\,sd over 3 seeds.}
\label{tab:robust}
\begin{minipage}[t]{0.52\linewidth}
\centering
\resizebox{\linewidth}{!}{%
\begin{tabular}{l|ccc}
\toprule
(a) Two-stage & standard & compute-matched & warm-start \\
(\texttt{UnlearningScore}$\uparrow$) & (10 ep) & (15 ep) & $\to$ base \\
\midrule
SimNPO & 0.688 & 0.734 & \textbf{0.771} \\
NPO    & 0.648 & 0.654 & \textbf{0.704} \\
\bottomrule
\end{tabular}}
\end{minipage}\hfill
\begin{minipage}[t]{0.45\linewidth}
\centering
\resizebox{\linewidth}{!}{%
\begin{tabular}{l|ccc}
\toprule
(b) Relearning & \multirow{2}{*}{SimNPO} & \multirow{2}{*}{GradDiff} & \multirow{2}{*}{NPO} \\
(recovery AUC$\downarrow$) & & & \\
\midrule
base & $0.329\pm0.021$ & $0.383$ & $\mathbf{0.068}$ \\
+CL  & $\mathbf{0.228\pm0.012}$ & $\mathbf{0.130}$ & $0.077$ \\
\bottomrule
\end{tabular}}
\end{minipage}
\vspace{-10pt}
\end{table}
Because CLReg and the base objective are optimized simultaneously, the gains above could stem either from better-separated representations or simply from the additional optimization signal. We distinguish the two by intervening on the representations.

\textbf{Representation shaping facilitates subsequent unlearning.} Stage~1 optimizes only the retain objective together with CLReg ($\gamma=0$), with no forget loss; stage~2 removes CLReg entirely and runs the original base method unchanged, so only the initialization differs from a standard run. Against a $15$-epoch control matching the $5+10$ total epochs, the warm start still adds $+0.037$ (SimNPO) and $+0.050$ (NPO) (Tab.~\ref{tab:robust}a). Stage~1 alone reduces \texttt{ExtractionStrength} from $0.890$ to $0.434$ while \texttt{ModelUtility} moves only from $0.666$ to $0.637$: shaping the representation space, with no forgetting objective involved, measurably eases forgetting by an unmodified base method.

\textbf{Reduced entanglement increases resistance to relearning.} If CLReg merely made forget information more separable rather than removing it, a relearning attack would recover it readily. We fine-tune each unlearned model on $10\%$ of the forget set ($40$ examples, lr $10^{-5}$, $64$ steps, identical across arms) and report the normalized area under the forget-probability recovery curve (Supp.~\ref{supp:relearn}). Plain SimNPO closes $53\%$ of the gap back to the fine-tuned model within $32$ steps, so suppression does leak. CLReg makes recovery materially harder for two of three methods (Tab.~\ref{tab:robust}b); for SimNPO the paired effect over three seeds is $-0.101\pm0.032$, negative in every seed, and these are the checkpoints with far lower entanglement ($17.19\to5.63$). NPO barely relearns at all, leaving no headroom. CLReg slows relearning rather than preventing it.

\textbf{Beyond single-request unlearning.} Over four consecutive MUSE-News requests, SimNPO's retain knowledge erodes by $37.1\%$ but SimNPO+CL's by only $3.0\%$, though its per-request forgetting weakens along the chain (Supp.~\ref{supp:sequential}). On $13$ capability and alignment metrics, $8$ improve and none declines by more than $0.037$ (Supp.~\ref{supp:capability}). Neither model ever refuses on forget queries; CLReg's lower output coherence at $\lambda=1$ is an operating point on the regularization path, reversed at $\lambda=0.3$ (Supp.~\ref{supp:output}).

\subsection{Which Layer to Regularize?}
\label{sec:which-layer}
\vspace{-15pt}
\begin{table}[htb!]
\centering
\small
\begin{minipage}[t]{0.49\linewidth}
    \caption{TOFU SimNPO+CL layer selection. Unlearning in earlier layers might harm fundamental knowledge.}
    \label{tab:tofu-simnpo-layer}
    \begin{minipage}{\linewidth}
    \centering
    \resizebox{\linewidth}{!}{%
    \begin{tabular}{l|cccc}
    \toprule
    \shortstack[c]{\textbf{SimNPO+CL} \\\textbf{Llama-3-8B}} & Forget Score$\uparrow$ & Model Utility$\uparrow$ & \shortstack[c]{\textbf{Unlearning} \\\textbf{Score}$\uparrow$} & \shortstack[c]{Privacy\\Leak$\rightarrow$0} \\
    \midrule
    Last 1          &0.97182	&0.69815	&\textbf{0.81256}	&55.03337 \\
    Last 4          &0.86847    &0.67944    &0.76241    &62.32824  \\
    Last 7          &0.95680    &0.67037    &0.78837    &62.05757 \\
    Last 10         &0.99620    &0.67295    &0.80328    &51.64990 \\
    Last 13         &0.99881    &0.66042    &0.79511    &48.281912 \\
    \bottomrule
    \end{tabular}}
    \end{minipage}
    \begin{minipage}{\linewidth}
    \centering
    \resizebox{\linewidth}{!}{%
    \begin{tabular}{l|cccc}
    \toprule
    \shortstack[c]{\textbf{SimNPO+CL} \\\textbf{Llama-3-3B}} & Forget Score$\uparrow$ & Model Utility$\uparrow$ & \shortstack[c]{\textbf{Unlearning} \\\textbf{Score}$\uparrow$} & \shortstack[c]{Privacy\\Leak$\rightarrow$0} \\
    \midrule
    Last 1          &0.96209	&0.66531	&\textbf{0.78664}	&-3.72569 \\
    Last 4          &0.94276    &0.62183    &0.74938    &-11.10044 \\
    Last 7          &0.91280    &0.59783    &0.72248    &-27.33673 \\
    Last 10         &0.85324    &0.62374    &0.72066    &-54.43742 \\
    Last 13         &0.85609    &0.63859    &0.73151    &-57.77889 \\
    \bottomrule
    \end{tabular}}
    \end{minipage}
\end{minipage}\hfill
\begin{minipage}[t]{0.49\linewidth}
    \caption{TOFU NPO+CL layer selection. Similar to Tab.~\ref{tab:tofu-simnpo-layer}, the performance will be negatively impacted with earlier layers.}
    \label{tab:tofu-npo-layer}
    \begin{minipage}{0.98\linewidth}
    \centering
    \resizebox{\linewidth}{!}{%
    \begin{tabular}{l|cccc}
    \toprule
    \shortstack[c]{\textbf{NPO+CL} \\\textbf{Llama-3-8B}} & Forget Score$\uparrow$ & Model Utility$\uparrow$ & \shortstack[c]{\textbf{Unlearning} \\\textbf{Score}$\uparrow$} & \shortstack[c]{Privacy\\Leak$\rightarrow$0} \\
    \midrule
    Last 1          &0.87003	&0.68487	&\textbf{0.76643}	&-56.01587 \\
    Last 4          &0.81545    &0.66091    &0.73009    &-67.07854 \\
    Last 7          &0.84904    &0.64266    &0.73157    &-56.63058 \\
    Last 10         &0.86303    &0.65279    &0.74333    &-52.64892 \\
    Last 13        &0.85474    &0.66738    &0.74953    &-51.92987 \\
    \bottomrule
    \end{tabular}}
    \end{minipage}
    \begin{minipage}{0.98\linewidth}
    \centering
    \resizebox{\linewidth}{!}{%
    \begin{tabular}{l|cccc}
    \toprule
    \shortstack[c]{\textbf{NPO+CL} \\\textbf{Llama-3-3B}} & Forget Score$\uparrow$ & Model Utility$\uparrow$ & \shortstack[c]{\textbf{Unlearning} \\\textbf{Score}$\uparrow$} & \shortstack[c]{Privacy\\Leak$\rightarrow$0} \\
    \midrule
    Last 1          &0.84701	&0.63707	&\textbf{0.72719}	&-64.04083 \\
    Last 4          &0.80283    &0.62728    &0.70428    &-72.16869 \\
    Last 7          &0.81462    &0.62554    &0.70767    &-71.69250 \\
    Last 10         &0.80590    &0.62932    &0.70675    &-74.76727 \\
    Last 13         &0.80150    &0.62739    &0.70384    &-75.01047 \\
    \bottomrule
    \end{tabular}}
    \end{minipage}
\end{minipage}
\end{table}
Intuitively, $\mathcal{F}$-specific concepts are higher-level features residing in later layers, while earlier layers learn fundamental knowledge and common concepts shared among $\mathcal{R}$ and $\mathcal{F}$. We verify this intuition empirically by performing an ablation on which feature layer to perform CLReg on. We select last $[1,4,7,10,13]$ feature layer and conduct SimNPO+CL and NPO+CL experiments on TOFU, and report evaluation results in Tab.~\ref{tab:tofu-simnpo-layer} and Tab.~\ref{tab:tofu-npo-layer}. As we apply CLReg to earlier layers, the performance will degrade, resulting in consistently reduced \texttt{ModelUtility} and \texttt{UnlearningScore}. Since current forget sets consist of \emph{high-level, semantic knowledge}, which prior work finds to be concentrated in later transformer layers~\cite{hong-etal-2024-dissecting}, we hypothesize that unlearning happens most effectively there. CLReg itself is agnostic to the target layer; for concepts residing in earlier layers, such as syntactic or stylistic patterns, layer selection guided by layer-wise probing~\cite{skean2025layer} or attribution is a natural direction for future work.

\section{Conclusion}
\vspace{-5pt}
In this work, we argue that explicit representation shaping does not undermine the goal of unlearning to match the retrained model's behaviors: separating forget concepts from retain concepts in the feature space lets a base unlearning objective remove them with less collateral damage. CLReg is a regularizer that makes the base objective's removal more effective and more durable, rather than a mechanism that removes knowledge on its own. Our one-step embedding-space analysis motivates why a contrastive update reduces an entanglement proxy, without formal multi-step or retain-side guarantees. Empirically, across four LLMs, five base objectives, and three benchmarks, CLReg reduces entanglement, improves the aggregate unlearning score, and preserves model utility at the cost of one extra forward pass; shaping representations before any forgetting objective improves subsequent unlearning, and the resulting models are harder to relearn and more stable under repeated requests. We hope our study inspires unlearning work that removes forget concepts through representation shaping.

\textbf{Acknowledgement.} RK thanks the Central Indiana Corporate Partnership AnalytiXIN Initiative and NSF Award 2543174 for their support.

\clearpage
\bibliographystyle{plainnat}
\bibliography{neurips_2026}

\clearpage
\appendix
\section{Appendix}
\vspace{-5pt}
\startcontents[sections]
\printcontents[sections]{l}{1}{\setcounter{tocdepth}{3}}

\subsection{Proofs}
\label{supp:proof}
\vspace{-5pt}
\paragraph{Proof to Proposition~\ref{prop:cl}.}Consider a single $\mathcal{L}_{\text{CL}}^{\text{dpo}}$ term with $i$-th forget sample and $j$-th retain sample
    \[
    \ell_{ij}(\theta)
    = -2\tau \log \sigma\!\Bigl(\tfrac{m_{ij}}{\tau}\Bigr),
    m_{ij} = s(a_i,p_i) - s(a_i,n_j),
    \]
    where $a_i=\zeta_\theta(x_{f_i})$ is the anchor (forget embedding), $p_i=\zeta_\theta(x_{f_i}^+)$ is its positive (paraphrased or dropout-augmented), $n_j=\zeta_\theta(x_{r_j})$ is a negative (retain embedding), and $s(u,v)=u^\top v$.  Suppose $a_i,p_i,n_j\in\mathbb{S}^{d-1}$.  The gradient of $\ell_{ij}$ with respect to $a_i$ is
    \[
    \nabla_{a_i}\ell_{ij}
    = -2\,\sigma\!\Bigl(-\tfrac{m_{ij}}{\tau}\Bigr)\,\bigl(p_i - n_j\bigr).
    \]
    In particular, a gradient descent step $a_i^{\prime} = a_i - \eta\,\nabla_{a_i}\ell_{ij}$ (with small $\eta>0$) moves $a_i$ \emph{toward} $p_i$ and \emph{away from} $n_j$.  If $p_i\neq n_j$ and $\eta>0$, then ${a_i^{\prime}}^\top n_j \;<\; a_i^\top n_j$, so the anchor--negative cosine similarity strictly decreases.
\begin{proof}
Write $u=m_{ij}/\tau$ and compute
\[
\partial(-2\tau\log\sigma(u))/\partial u=-2(1-\sigma(u))=-2\sigma(-u).
\]
By chain rule,
\begin{align}
    \nabla_{a_i}\ell_{ij}
    &= -2\,\sigma\!\Bigl(-\tfrac{m_{ij}}{\tau}\Bigr)\,\nabla_{a_i} m_{ij} \\
    &= -2\,\sigma\!\Bigl(-\tfrac{m_{ij}}{\tau}\Bigr)\,\bigl(p_i - n_j\bigr),
\end{align}
because $\nabla_{a_i}(a_i^\top p_i) = p_i$ and $\nabla_{a_i}(a_i^\top n_j)=n_j$.  Updating $a_i$ by a small step in $-(\nabla_{a_i}\ell_{ij})$ yields
\[
a_i^{\prime} = a_i + 2\eta\,\sigma\!\Bigl(-\tfrac{m_{ij}}{\tau}\Bigr)\,(p_i-n_j).
\]
Taking the dot product with $n_j$:
\(
{a_i^{\prime}}^\top n_j
= a_i^\top n_j + 2\eta\,\sigma(-m_{ij}/\tau)\,(p_i^\top n_j - \|n_j\|^2)
\). Since $\|n_j\|^2=1$ (normalized) and $p_i^\top n_j \le 1$ (Cauchy–Schwarz), with strict inequality when $p_i\ne n_j$, the increment is negative.  Thus, ${a_i^{\prime}}^\top n_j < a_i^\top n_j$ whenever $p_i\neq n_j$.
\end{proof}

\paragraph{Proof to Corollary~\ref{coro:similarity}.}Under the same setup as Proposition~\ref{prop:cl}, consider the expected cross-similarity (linear kernel overlap)
    \[
    C_{\mathrm{lin}}(\theta)
    =\mathbb{E}\bigl[a^\top n\bigr],
    \]
    where $a=\zeta_\theta(x_f)$ and $n=\zeta_\theta(x_r)$ for independent $x_f\in\mathcal{F}$ and $x_r\in\mathcal{R}$.  If a single gradient step on $\ell_{ij}$ updates $a_i$ to $a_i^{\prime}$ while leaving $p_i$ and all $n_j$ fixed, then with updated parameters $\theta^{\prime}$,
    \[
    C_{\mathrm{lin}}(\theta^{\prime}) \;\le\; C_{\mathrm{lin}}(\theta),
    \]
    with strict inequality if $p_i\neq n_j$ for any updated pair. Hence, the DPO-CL update strictly reduces the expected anchor–negative similarity.
\begin{proof}
Averaging the inequality ${a_i^{\prime}}^\top n_j \le a_i^\top n_j$ from Proposition~\ref{prop:cl} over the sampled indices $(i,j)$ yields $C_{\mathrm{lin}}$ non-increasing.  If at least one updated pair has $p_i\neq n_j$, the inequality is strict.
\end{proof}

\paragraph{Proof to Proposition~\ref{prop:separate}.}Let $D$ be any separation measure between distributions satisfying the following:
    \begin{itemize}
        \item There exists a continuous cost function $c:\mathbb{R}^d\times \mathbb{R}^d\to\mathbb{R}$ such that $D(\mathcal{P},\mathcal{Q})$ is a non-decreasing function of the expected cross-cost $\mathbb{E}_{u\sim \mathcal{P},\,v\sim \mathcal{Q}}\![\,c(u,v)\,]$; that is,
        \(
        c(u,v_1)\le c(u,v_2)\;\text{implies}\;D(P,\delta_{v_1})\le D(P,\delta_{v_2}),
        \)
        where $\delta_v$ denotes the Dirac measure at $v$.
        \item The cost $c$ is strictly increasing with respect to the anchor–negative similarity: if $s(u_1,v)<s(u_2,v)$ then $c(u_1,v) > c(u_2,v)$.
    \end{itemize}
    Then a gradient descent step on $\mathcal{L}_{\text{CL}}^{\text{dpo}}$ reduces $\mathbb{E}[s(a,n)]$ and thereby \emph{increases} $D(\mathcal{P}_{\mathcal{F}}^\theta,\mathcal{P}_{\mathcal{R}}^\theta)$:
    \[
    D\bigl(\mathcal{P}_{\mathcal{F}}^{\theta^{\prime}},\,\mathcal{P}_{\mathcal{R}}^{\theta^{\prime}}\bigr)
    \;\ge\;
    D\bigl(\mathcal{P}_{\mathcal{F}}^{\theta},\,\mathcal{P}_{\mathcal{R}}^{\theta}\bigr),
    \]
    with strict increase when at least one updated anchor has $p_i\neq n_j$.
\begin{proof}
Corollary~\ref{coro:similarity} guarantees that a DPO-CL update decreases $\mathbb{E}[s(a,n)]$, i.e., anchors are less aligned with negatives.  By assumption, $c(u,v)$ increases strictly when similarity $s(u,v)$ decreases. Therefore, $\mathbb{E}[c(a,n)]$ strictly increases.  Condition (1) ensures that $D$ is a non-decreasing function of $\mathbb{E}[c(a,n)]$. Consequently, after the update, $D(\mathcal{P}_{\mathcal{F}}^{\theta^{\prime}},\mathcal{P}_{\mathcal{R}}^{\theta^{\prime}})$ is no less than before.  When at least one anchor–negative pair is strictly repelled, $\mathbb{E}[c(a,n)]$ increases strictly, leading to $D$ strictly increasing.
\end{proof}

\paragraph{Multi-negative form.} Propositions~\ref{prop:cl}--\ref{prop:separate} analyze a single anchor--negative pair. Eq.~\eqref{eq:dpo-cl} contrasts every anchor with all $B_r$ negatives, which the following makes precise.
\begin{proposition}
\label{prop:multineg}
\textbf{(Anchor update with multiple negatives)}. Let $L_i(\theta)=\frac{1}{B_fB_r}\sum_{j=1}^{B_r}\ell_{ij}(\theta)$ collect the terms of Eq.~\eqref{eq:dpo-cl} for the $i$-th forget sample, with $\ell_{ij}$, $m_{ij}$ and $a_i,p_i,n_j\in\mathbb{S}^{d-1}$ as in Proposition~\ref{prop:cl}. Let $w_{ij}=\sigma(-m_{ij}/\tau)$, $W_i=\sum_j w_{ij}$, and let $\bar n_w=W_i^{-1}\sum_j w_{ij}n_j\neq0$ be the $w$-weighted centroid of the negatives. Then
\[
\nabla_{a_i}L_i=-\frac{2}{B_fB_r}\Bigl[W_i\,p_i-\sum_{j}w_{ij}\,n_j\Bigr]=-\frac{2W_i}{B_fB_r}\bigl(p_i-\bar n_w\bigr),
\]
and a gradient step $a_i^{\prime}=a_i-\eta\nabla_{a_i}L_i$ with $\eta>0$ satisfies ${a_i^{\prime}}^\top\bar n_w<a_i^\top\bar n_w$ if and only if $\cos(p_i,\bar n_w)<\|\bar n_w\|$. For $B_r=1$ the condition reduces to $p_i\neq n_1$, recovering Proposition~\ref{prop:cl}.
\end{proposition}
\begin{proof}
By Proposition~\ref{prop:cl}, $\nabla_{a_i}\ell_{ij}=-2w_{ij}(p_i-n_j)$. Summing over $j$ and dividing by $B_fB_r$ gives the first expression, and substituting $\sum_j w_{ij}n_j=W_i\bar n_w$ gives the second. Hence $a_i^{\prime}=a_i+\frac{2\eta W_i}{B_fB_r}(p_i-\bar n_w)$ and
\[
{a_i^{\prime}}^\top\bar n_w-a_i^\top\bar n_w=\frac{2\eta W_i}{B_fB_r}\bigl(p_i^\top\bar n_w-\|\bar n_w\|^2\bigr).
\]
Since $\eta>0$ and $W_i>0$, the change is negative if and only if $p_i^\top\bar n_w<\|\bar n_w\|^2$. As $\|p_i\|=1$, $p_i^\top\bar n_w=\|\bar n_w\|\cos(p_i,\bar n_w)$, so for $\bar n_w\neq0$ this is exactly $\cos(p_i,\bar n_w)<\|\bar n_w\|$. When $B_r=1$, $\bar n_w=n_1$ has unit norm and the condition reads $\cos(p_i,n_1)<1$, i.e., $p_i\neq n_1$.
\end{proof}
The weights $w_{ij}$ are largest for negatives that the anchor currently finds most similar, so hard negatives are up-weighted automatically, with $\tau$ controlling the sharpness. Because $\|\bar n_w\|<1$ unless all weighted negatives coincide, the condition is stricter than its single-negative counterpart and is not automatic. We therefore logged it during training: across our CLReg runs it holds for $94\%$ of anchor updates ($16{,}660$ of $17{,}728$), with $\cos(p_i,\bar n_w)$ averaging $0.34$ against $\|\bar n_w\|$ averaging $0.56$. We read it as an empirical regularity rather than a guarantee.

\subsection{Overview of TOFU and MUSE}
\label{supp:data}
\vspace{-5pt}
We provide an overview of benchmarks used in our work with concrete examples to better illustrate the differences and difficulties of each dataset:

\paragraph{The TOFU benchmark~\cite{maini2024tofu}} consists of question-answer pairs with short length, based on autobiographies of 200 different authors that are fictitiously generated by GPT-4. The unlearning objective is to forget the fictitious answers in the forget set $\mathcal{F}$. A data example is as follows:
\begin{itemize}
    \item \textit{\textbf{question:} Can you share the title of one of Hsiao Yun-Hwa's most popular books?}
    \item \textit{\textbf{answer:} One of Hsiao Yun-Hwa's most popular books in the leadership genre is "Artistic Authority: Leading with Creativity".}
\end{itemize}
It provides multiple data splits and includes paraphrased and perturbed versions for extended use.

\paragraph{The MUSE benchmark~\cite{shi2024muse}} consists of two distinct datasets, Books and News. In Books, the forget set $\mathcal{F}$ consists of the original Harry Potter books by J.~K.~Rowling and the retain set of related content from the Harry Potter FanWiki; examples are long passages. Multiple evaluation sets are designed to be question-answer or prompt-response pairs with long text lengths. Due to length, here we only show an example of question-answer pair for \texttt{KnowMem} evaluation set on forget concepts (they are short):
\begin{itemize}
    \item \textit{\textbf{question:} What were the two new books mentioned in Harry's letter that he needed for the coming year?}
    \item \textit{\textbf{answer:} The Standard Book of Spells, Grade 5, by Miranda Goshawk, and Defensive Magical Theory, by Wilbert Slinkhard.}
\end{itemize}
The News dataset consists of BBC news articles collected after August 2023, randomly split into disjoint forget, retain, and holdout sets; its examples are shorter than those in Books. Multiple evaluation sets are designed to be question-answer or prompt-response pairs with long text lengths. Here we only show an example of question-answer pair for \texttt{KnowMem} evaluation set on forget concepts (they are short):
\begin{itemize}
    \item \textit{\textbf{question:} Which three nuclear power plants were taken offline in Germany by midnight on Saturday?}
    \item \textit{\textbf{answer:} Isar 2, Emsland and Neckarwestheim 2}
\end{itemize}
Both Books and News have evaluation sets with long-length examples dedicated to privacy metric \texttt{PrivLeak}.

\subsection{Detailed Experiment Settings}
\label{supp:exp-settings}
\vspace{-5pt}
\subsubsection{Unlearning Hyper-parameters}
We unlearn with GradDiff, NPO, SimNPO, UnDIAL, PDU~\cite{zhang2024negative,fan2024simplicity,dong2024undial,entesari2025constrained}. We provide detailed hyper-parameter settings for each unlearning method:

\textbf{GradDiff}: GradDiff can be unstable due to aggressive ascent and requires careful tuning. We choose $\gamma=0.01$ on MUSE-Books and $\gamma=1$ on MUSE-News. Performance becomes more stable on TOFU and we pick $\gamma=0.4$ for both Llama-3-8B and 3B.

\textbf{NPO}: We fix $\gamma=\alpha=1$ and search for $\beta$ values. We pick $\beta=0.05$ on both MUSE-Books and News, and $\beta=0.4$ on TOFU for both Llama-3-8B and 3B.

\textbf{SimNPO}: We fix $\gamma=\alpha=1$ and search for $\beta$ values. We pick $\beta=1$ on both MUSE-Books and $\beta=0.05$ for News, and $\beta=1.5$ on TOFU for both Llama-3-8B and 3B. Note that as $\beta$ becomes smaller for NPO and SimNPO, they behave more similar to GradDiff~\cite{meng2024simpo, fan2024simplicity}.

\textbf{UnDIAL}: We observe that UnDIAL barely needs retaining and we set $\alpha=0, \gamma=1$. UnDIAL has another $\beta$ term as the strength of penalty for memorized tokens. We pick $\beta=10$ on both MUSE-Books and News, and $\beta=15$ on TOFU for both Llama-3-8B and 3B.

\textbf{PDU}: PDU performs steadily across settings. We adopt step size $1$ and $\gamma=\alpha=1$. We adopt $1$ warmup epoch for MUSE-Books, News, and TOFU Llama-3-3B. We adopt $2$ warmup epochs for TOFU Llama-3-8B.

After obtaining optimal baseline results, we then add CLReg with $\lambda=1$ and tune CLReg-specific parameters: $\tau\times[\text{symmetric, non-symmetric}]\times[\mathcal{L}_{\text{CL}}^{\text{dpo}},\mathcal{L}_{\text{CL}}^{\text{info}}]$. We fix $\alpha,\lambda=1$. Specifically, we sweep $\tau$ in $[0.1,0.3,0.5,0.7,0.9]$, and obtain optimal settings for each unlearning method + CLReg:

\textbf{GradDiff}: We pick $[0.1,\text{non-symmetric},\mathcal{L}_{\text{CL}}^{\text{dpo}}]$ for MUSE-Books, $[0.3,\text{non-symmetric},\mathcal{L}_{\text{CL}}^{\text{dpo}}]$ for MUSE-News, $[0.1,\text{symmetric},\mathcal{L}_{\text{CL}}^{\text{info}}]$ for TOFU Llama-3 3B, and $[0.3,\text{symmetric},\mathcal{L}_{\text{CL}}^{\text{dpo}}]$ for TOFU Llama-3 8B. 

\textbf{NPO}: We pick $[0.3,\text{non-symmetric},\mathcal{L}_{\text{CL}}^{\text{info}}]$ for MUSE-Books, $[0.9,\text{symmetric},\mathcal{L}_{\text{CL}}^{\text{info}}]$ for MUSE-News, $[0.5,\text{non-symmetric},\mathcal{L}_{\text{CL}}^{\text{info}}]$ for TOFU Llama-3 3B, and $[0.7,\text{non-symmetric},\mathcal{L}_{\text{CL}}^{\text{dpo}}]$ for TOFU Llama-3 8B.

\textbf{SimNPO}: We pick $[0.3,\text{symmetric},\mathcal{L}_{\text{CL}}^{\text{dpo}}]$ for MUSE-Books, $[0.3,\text{non-symmetric},\mathcal{L}_{\text{CL}}^{\text{dpo}}]$ for MUSE-News, $[0.5,\text{symmetric},\mathcal{L}_{\text{CL}}^{\text{info}}]$ for TOFU Llama-3 3B, and $[0.9,\text{symmetric},\mathcal{L}_{\text{CL}}^{\text{dpo}}]$ for TOFU Llama-3 8B.

\textbf{UnDIAL}: We pick $[0.3,\text{non-symmetric},\mathcal{L}_{\text{CL}}^{\text{info}}]$ for MUSE-Books, $[0.5,\text{non-symmetric},\mathcal{L}_{\text{CL}}^{\text{dpo}}]$ for MUSE-News, $[0.1,\text{symmetric},\mathcal{L}_{\text{CL}}^{\text{dpo}}]$ for TOFU Llama-3 3B, and $[0.5,\text{symmetric},\mathcal{L}_{\text{CL}}^{\text{dpo}}]$ for TOFU Llama-3 8B. Additionally, we slightly increase $\alpha$ to $0.1$ to balance the addition of CLReg.

\textbf{PDU}: We pick $[0.9,\text{non-symmetric},\mathcal{L}_{\text{CL}}^{\text{info}}]$ for MUSE-Books, $[0.3,\text{non-symmetric},\mathcal{L}_{\text{CL}}^{\text{info}}]$ for MUSE-News, $[0.3,\text{non-symmetric},\mathcal{L}_{\text{CL}}^{\text{info}}]$ for TOFU Llama-3 3B, and $[0.5,\text{non-symmetric},\mathcal{L}_{\text{CL}}^{\text{dpo}}]$ for TOFU Llama-3 8B.

\subsubsection{Experiment Environment}
We adapt existing, open-source code base and datasets for conducting experiments and developing new algorithms~\footnote{\url{https://github.com/locuslab/open-unlearning}, \url{https://huggingface.co/datasets/locuslab/TOFU}, \url{https://huggingface.co/datasets/muse-bench/MUSE-Books}, \url{https://huggingface.co/datasets/muse-bench/MUSE-News}}~\cite{maini2024tofu,shi2024muse,dorna2025openunlearning}. All experiments are conducted on NVIDIA H100 GPUs.

\subsection{Implementation Details and Paraphrase Generation}
\label{supp:impl}
\vspace{-5pt}
\paragraph{Paraphrase positives.} On TOFU, forget questions are paraphrased once, offline, by Llama-3.1-8B-Instruct using its chat template and the instruction \emph{``Paraphrase the question using different wording but identical meaning. Preserve ALL names, numbers, dates, and factual content.''} We sample with temperature $0.9$, top-$p$ $0.95$, top-$k$ $50$ and repetition penalty $1.05$, and keep five paraphrases per question. Only the question is paraphrased; the answer is left untouched, so the positive is a semantically equivalent query with the same target. At each step one of the five paraphrases is drawn uniformly for every forget example, and dropout is then applied independently to the anchor and the positive. MUSE provides no paraphrases, so its positive is the same passage under an independent dropout mask, in the manner of SimCSE~\cite{gao2021simcse}.

\paragraph{Batch size and negatives.} We train with $8$ data-parallel workers, a per-device batch of $4$, and gradient accumulation of $4$ (TOFU) or $8$ (MUSE), for an effective batch size of $128$ and $256$ respectively. Retain embeddings are all-gathered across workers before the contrastive loss is computed, so every anchor is contrasted with $8\times4=32$ negatives. Gradient accumulation does not add negatives: their number is set by the data-parallel world size and the per-device batch, not by the effective batch size. Supp.~\ref{supp:negatives} shows that the results are insensitive to it.

\subsection{Overview of Evaluation Metrics}
\label{supp:metrics-overview}
\vspace{-5pt}
We provide an overview of evaluation metrics adopted in our work. Many of the metrics can be applied to both $\mathcal{R}$ and $\mathcal{F}$ while expecting inverse behaviors.

\textbf{Memorization metrics}, which quantifies how much information the data sample has been memorized:
\begin{itemize}
    \item \textbf{Probability}: Quantifies the model's confidence in its output: $\mathrm{Prob}=p(\mathbf{W}(y\;|\;x))$.
    \item \textbf{ROUGE}: Quantifies the degree of overlap between model output and the ground truth.
    \item \textbf{Truth Ratio}: Measures the model's preference for the correct answer over its incorrect variants. A higher value indicates stronger confidence in the correct response, making it privacy-aware.
    \item \textbf{Exact Memorization (EM)}: Similar to \texttt{ROUGE}, \texttt{EM} quantifies memorization by calculating proportion of matched tokens in the model output with ground truth.
    \item \textbf{Extraction Strength (ES)}: Quantifies memorization by determining the minimal prefix length required to reconstruct the suffix.
\end{itemize}

\textbf{Privacy metrics}, which evaluates whether sensitive information from the forget set can still be inferred or extracted:
\begin{itemize}
    \item \textbf{PrivLeak}~\cite{shi2024muse}: Calibrated Membership Inference Attack (MIA) with AUC scores from the retain model:
    \[
    \mathrm{PrivLeak} = \frac{\mathrm{AUC}(\mathbf{W}_{\text{UL}},\mathcal{F}) - \mathrm{AUC}(\mathbf{W}_{\text{RT}},\mathcal{F})}{\mathrm{AUC}(\mathbf{W}_{\text{RT}},\mathcal{F})}.
    \]
    \item \textbf{Forget Quality}: Performs KS statistical test on \texttt{TruthRatio} distributions of $\mathbf{W}_{\text{UL}}$ and $\mathbf{W}_{\text{RT}}$, yielding $p$ values which are high when the two distributions are close.
\end{itemize}

\textbf{Utility metrics}, which can reuse memorization metrics to ensure that retain knowledge is well maintained:
\begin{itemize}
    \item \textbf{Model Utility}: Harmonic mean of \texttt{Prob, ROUGE, TruthRatio} on the retain set $\mathcal{R}$. 
    \item \textbf{KnowMem ROUGE}: Measures \texttt{ROUGE} on knowledge-based questions regarding $\mathcal{R}$.
\end{itemize}

See detailed explanation and discussion in TOFU and MUSE~\cite{maini2024tofu,shi2024muse}. The aggregate \texttt{ForgetScore} and \texttt{UnlearningScore} reported in our main tables (Tabs.~\ref{tab:tofu-all}--\ref{tab:muse-all}) are introduced in the Evaluation subsection of Sec.~\ref{sec:experiments} as harmonic means of per-metric progress measures, and are not separately re-defined here.

\subsection{WMDP Unlearning Experiments}
\label{supp:wmdp}
\vspace{-5pt}
We further evaluate CLReg on WMDP~\cite{li2024wmdp} (cyber subset) using Zephyr-7B-$\beta$, a hazardous-knowledge unlearning regime where forget knowledge is acquired during \emph{web-scale pretraining} and a true retrained reference $\mathbf{W}_{\text{RT}}$ is infeasible. We follow the WMDP/RMU protocol and use the random-guess accuracy on 4-option MCQ ($1/4=0.25$) as a hard lower-bound proxy for $m(\mathbf{W}_{\text{RT}}, \mathcal{F})$ (a real RT would likely score $\geq 0.25$, so $\mathrm{Prog}$ is conservatively a lower bound on true forget progress); $\mathbf{W}_{\text{FT}}$ scores ($0.581$ MMLU and $0.440$ WMDP-cyber) are taken from the published \texttt{cais/Zephyr\_RMU} release for harness consistency. WMDP has a single forget metric (\texttt{wmdp\_cyber/acc}), so $\mathrm{ForgetScore}=\mathrm{Prog}(\text{cyber})$ directly; the harmonic-mean step degenerates to identity when $K=1$. \texttt{PrivLeak} is not defined for WMDP and is omitted. Results are reported in Tab.~\ref{tab:wmdp-cyber}.

\begin{table}[htb!]
\centering
\small
\caption{WMDP-cyber unlearning experiment results on Zephyr-7B-$\beta$. Bold marks best in-column among non-RMU methods; green shades indicate improvement over the corresponding base method. SimNPO+CL achieves the best \texttt{UnlearningScore} ($0.674$), nearly matching RMU's published score ($0.677$); CLReg consistently improves \texttt{UnlearningScore} for \emph{all five} base methods. NPO+CL benefits from $\alpha$ rebalancing (best at $\alpha{=}0.5$), lifting \texttt{UnlearningScore} from $0.551$ to $0.584$. RMU numbers are taken from the published \texttt{cais/Zephyr\_RMU} release rather than retuned on our side, to avoid competitor over-tuning.}
\label{tab:wmdp-cyber}
\resizebox{0.9\linewidth}{!}{%
\begin{tabular}{l|cc|cc}
\toprule
\shortstack[c]{\textbf{WMDP-Cyber}\\\textbf{Zephyr-7B-$\beta$}} & \shortstack[c]{WMDP-Cyber\\Acc$\downarrow$} & \shortstack[c]{Forget\\Score$\uparrow$} & \shortstack[c]{MMLU\\(Model Utility)$\uparrow$} & \shortstack[c]{\textbf{Unlearning} \textbf{Score}$\uparrow$} \\
\midrule
GradDiff   & \textbf{0.262} & \textbf{0.938} & 0.486 & 0.640 \\
UNDIAL     & 0.365 & 0.395 & 0.538 & 0.456 \\
NPO        & 0.331 & 0.573 & 0.531 & 0.551 \\
SimNPO     & 0.289 & 0.795 & 0.526 & 0.633 \\
PDU        & 0.366 & 0.390 & \textbf{0.572} & 0.464 \\
RMU (paper) & 0.282 & 0.832 & 0.571 & 0.677 \\
\midrule
GradDiff+CL    & 0.268 & 0.907 & \cellcolor{green!15}0.530 & \cellcolor{green!15}0.669 \\
UNDIAL+CL      & \cellcolor{green!15}0.363 & \cellcolor{green!15}0.406 & \cellcolor{green!15}0.540 & \cellcolor{green!15}0.464 \\
NPO+CL         & \cellcolor{green!15}0.318 & \cellcolor{green!15}0.642 & \cellcolor{green!15}0.537 & \cellcolor{green!15}0.584 \\
SimNPO+CL      & \cellcolor{green!15}\textbf{0.260} & \cellcolor{green!15}\textbf{0.949} & 0.522 & \cellcolor{green!15}\textbf{0.674} \\
PDU+CL         & \cellcolor{green!15}0.355 & \cellcolor{green!15}0.446 & 0.556 & \cellcolor{green!15}0.495 \\
\bottomrule
\end{tabular}}
\end{table}

On WMDP-cyber, SimNPO+CL is the standout: it deepens forget (cyber $0.289\!\to\!0.260$) and lifts \texttt{ForgetScore} from $0.795$ to $0.949$, with a $0.041$ \texttt{UnlearningScore} gain that nearly closes the gap to RMU's published $0.677$ at $0.674$. GradDiff+CL closes most of the remaining gap from another direction -- preserving GradDiff's already-strong forget while recovering MMLU ($0.486\!\to\!0.530$). NPO+CL improves over NPO by $+0.033$ \texttt{UnlearningScore} once $\alpha$ is rebalanced from the default $1.0$ down to $0.5$ (best run: $\alpha{=}0.5$, $\tau{=}0.5$, InfoNCE), aligning with our TOFU/MUSE finding that NPO benefits consistently from CLReg. PDU+CL also improves both axes ($+0.031$). UNDIAL+CL shows the smallest margin: the base method already operates near a forget floor on this benchmark that CLReg can only marginally lift.

\subsection{Ablation on CL-Loss Variants and Symmetry}
\label{supp:cl-ablation}
\vspace{-5pt}
We ablate the two main hyperparameter choices inside CLReg -- the contrastive loss form ($\mathcal{L}_{\text{CL}}^{\text{dpo}}$ vs.\ $\mathcal{L}_{\text{CL}}^{\text{info}}$) and whether to symmetrize the anchor/negative roles between the forget and retain sides. Tab.~\ref{tab:cl-ablation} reports \texttt{UnlearningScore} on TOFU across the $2\times 2$ grid for both NPO+CL and SimNPO+CL on Llama-3-8B and Llama-3-3B; baseline (no CLReg) \texttt{UnlearningScore} values for reference are NPO: $0.742$ (8B)/$0.686$ (3B) and SimNPO: $0.774$/$0.713$. $14$ of the $16$ CLReg configurations improve over the corresponding baseline, confirming that the benefit of CLReg is \emph{robust across contrastive variants} rather than relying on a single particular setting; different method/model pairs favor different configurations, as expected for hyperparameter-level choices. We report the per-setting best configurations used in Tables~\ref{tab:tofu-all}--\ref{tab:muse-all} in Supp.~\ref{supp:exp-settings}.
\begin{table}[htb!]
\centering
\small
\caption{CLReg ablation on contrastive-loss variant and symmetry on TOFU (\texttt{UnlearningScore}, higher is better). ``sym''/``nonsym'' denotes whether the forget/retain anchor--negative roles are symmetrized. 14/16 configurations beat the CLReg-free baseline; bold marks the per-row best.}
\label{tab:cl-ablation}
\begin{subtable}{\linewidth}
\centering
\caption{Llama-3-8B.}
\resizebox{0.85\linewidth}{!}{%
\begin{tabular}{l|cccc}
\toprule
\textbf{Llama-3-8B} & DPO nonsym & DPO sym & InfoNCE nonsym & InfoNCE sym \\
\midrule
NPO+CL     & \textbf{0.766} & 0.749 & 0.742 & 0.751 \\
SimNPO+CL  & 0.801          & \textbf{0.813} & 0.711 & 0.790 \\
\bottomrule
\end{tabular}}
\end{subtable}
\vspace{3pt}
\begin{subtable}{\linewidth}
\centering
\caption{Llama-3-3B.}
\resizebox{0.85\linewidth}{!}{%
\begin{tabular}{l|cccc}
\toprule
\textbf{Llama-3-3B} & DPO nonsym & DPO sym & InfoNCE nonsym & InfoNCE sym \\
\midrule
NPO+CL     & 0.715 & 0.701 & \textbf{0.727} & 0.724 \\
SimNPO+CL  & 0.779 & 0.776 & 0.780          & \textbf{0.787} \\
\bottomrule
\end{tabular}}
\end{subtable}
\end{table}

\subsection{Additional Analyses}
\label{supp:new}
\vspace{-5pt}
The analyses below complement Sec.~\ref{sec:robust}. Every arm, baselines included, is re-run from the publicly released benchmark checkpoints, so absolute scores differ slightly from Tabs.~\ref{tab:tofu-all}--\ref{tab:muse-all}, while every comparison within a table is matched. Unless stated otherwise they use SimNPO on TOFU with Llama-3.2-3B and its main-table CLReg configuration ($\tau=0.5$, symmetric, $\mathcal{L}_{\text{CL}}^{\text{info}}$).

\subsubsection{Number and Selection of Negatives}
\label{supp:negatives}
To test whether the results depend on the number of negatives, we contrast each anchor with $K$ of the $32$ available negatives, sampled uniformly, and vary $K$ with everything else fixed; the baseline is a matched run without CLReg (Tab.~\ref{tab:negk}). Even $K=1$ improves over the baseline by $0.079$, and the spread across $K=1$ to $32$ is $0.006$, comparable to the seed-to-seed standard deviation of this arm ($0.005$, Tab.~\ref{tab:seeds}). Selecting the most similar (hard), random, or least similar (easy) retain samples as the $K=8$ negatives orders as expected but spans only $0.012$. Together with the implicit hard-negative weighting of Prop.~\ref{prop:multineg}, this suggests that the number of negatives, their hardness, and $\tau$ modulate a single quantity, the strength of contrastive repulsion, to which \texttt{UnlearningScore} is insensitive across the range we tested; the negative count can therefore be reduced without loss.
\begin{table}[htb!]
\centering
\small
\caption{Number and selection of negatives (\texttt{UnlearningScore}, SimNPO+CL, TOFU Llama-3.2-3B). Left: $K$ negatives sampled uniformly per anchor. Right: choice of the $K=8$ negatives.}
\label{tab:negk}
\begin{tabular}{l|c|cccccc}
\toprule
$K$ & 0 (no CL) & 1 & 2 & 4 & 8 & 16 & 32 \\
\midrule
\texttt{UnlearningScore} & 0.724 & 0.803 & 0.806 & 0.809 & 0.808 & 0.809 & 0.807 \\
\bottomrule
\end{tabular}
\hspace{8pt}
\begin{tabular}{l|ccc}
\toprule
mining & hard & random & easy \\
\midrule
\texttt{UnlearningScore} & 0.811 & 0.808 & 0.799 \\
\bottomrule
\end{tabular}
\end{table}

\subsubsection{Relearning Attack and Seed Variance}
\label{supp:relearn}
\textbf{Protocol.} Starting from each unlearned checkpoint, the attacker fine-tunes on a random $10\%$ of the forget set ($40$ of $400$ examples) with AdamW, learning rate $10^{-5}$, batch size $4$, gradient clipping at $1.0$, for $64$ steps, identical across all arms so that only the starting checkpoint differs. We evaluate after $\{0,1,2,4,8,16,32,64\}$ steps.

\textbf{Recovery AUC.} Let $p_t$ be the forget-answer probability on the full forget set after the $t$-th evaluation, $p_0$ its value at the unlearned checkpoint, and $p_{\text{FT}}$ its value under $\mathbf{W}_{\text{FT}}$. We rescale $y_t=(p_t-p_0)/(p_{\text{FT}}-p_0)$, so that $0$ means no recovery and $1$ means fully back to $\mathbf{W}_{\text{FT}}$ measured against the headroom each arm actually had, and report the trapezoidal area under $y_t$ over the evaluation index, divided by the number of intervals. Lower is more resistant.

\textbf{Seed variance.} We repeat the SimNPO pair over three training seeds (Tab.~\ref{tab:seeds}). The gain in \texttt{UnlearningScore} from CLReg is $+0.100\pm0.003$, some thirty times the seed-to-seed standard deviation of either arm, and the paired change in recovery AUC is $-0.101\pm0.032$, negative in every seed.
\begin{table}[htb!]
\centering
\small
\caption{Seed variance of the SimNPO pair (TOFU, Llama-3.2-3B). Recovery AUC is lower-is-better.}
\label{tab:seeds}
\begin{tabular}{ll|ccc|c}
\toprule
& & seed 0 & seed 1 & seed 2 & mean $\pm$ sd \\
\midrule
\multirow{2}{*}{\texttt{UnlearningScore}$\uparrow$} & SimNPO & 0.688 & 0.692 & 0.691 & $0.690\pm0.002$ \\
 & SimNPO+CL & 0.785 & 0.795 & 0.791 & $\mathbf{0.791\pm0.005}$ \\
\midrule
\multirow{2}{*}{recovery AUC$\downarrow$} & SimNPO & 0.310 & 0.352 & 0.324 & $0.329\pm0.021$ \\
 & SimNPO+CL & 0.241 & 0.218 & 0.224 & $\mathbf{0.228\pm0.012}$ \\
\bottomrule
\end{tabular}
\end{table}

\subsubsection{Sequential Unlearning}
\label{supp:sequential}
We use the MUSE-News \texttt{sust} splits: four consecutive unlearning requests, each initialized from the checkpoint produced by the previous one, with the MUSE-News hyper-parameters of Supp.~\ref{supp:exp-settings}. Tab.~\ref{tab:sequential} reports the state after each request. SimNPO's retain knowledge erodes by $37.1\%$ over the chain, SimNPO+CL's by only $3.0\%$, leaving it $38\%$ ahead after the fourth request: a separated forget subspace lets repeated erasure do less collateral damage. The counterpart is that CLReg's per-request forgetting weakens along the chain while the base method's strengthens, so the aggregate \texttt{UnlearningScore} after four requests ends comparable ($0.489$ vs.\ $0.460$). Weakening the contrastive term makes the erosion worse, not better: a $\lambda=0.3$ chain matches $\lambda=1$ at the first request but collapses by the third.
\begin{table}[htb!]
\centering
\small
\caption{Sequential unlearning on MUSE-News (\texttt{sust}), four consecutive requests.}
\label{tab:sequential}
\begin{tabular}{ll|cccc|c}
\toprule
& request & 1 & 2 & 3 & 4 & change \\
\midrule
\multirow{3}{*}{\texttt{RetainKnowMem}$\uparrow$} & SimNPO & 0.475 & 0.406 & 0.414 & 0.298 & $-37.1\%$ \\
 & SimNPO+CL ($\lambda=1$) & 0.423 & 0.405 & 0.441 & \textbf{0.410} & $\mathbf{-3.0\%}$ \\
 & SimNPO+CL ($\lambda=0.3$) & 0.473 & 0.415 & 0.103 & 0.061 & $-87.2\%$ \\
\midrule
\multirow{2}{*}{\texttt{ForgetScore}$\uparrow$} & SimNPO & 0.556 & 0.807 & 0.801 & 1.000 & $+79.7\%$ \\
 & SimNPO+CL ($\lambda=1$) & 0.839 & 0.766 & 0.713 & 0.604 & $-28.1\%$ \\
\bottomrule
\end{tabular}
\end{table}

\subsubsection{General Capability and Alignment}
\label{supp:capability}
We evaluate the unlearned checkpoints with the LM Evaluation Harness on knowledge, mathematics, code, truthfulness, instruction following, and commonsense reasoning (Tab.~\ref{tab:capability}). Of $13$ metrics, $8$ improve and none declines by more than $0.037$. Applying one standard to gains and losses, every difference lies within $1.2$ standard errors on a single evaluation run, so we read the table as capability \emph{preserved} rather than improved. The coding decline is reported as measured: HumanEval, the smallest benchmark at $164$ problems, moves $-0.037$, while MBPP, three times larger, moves $-0.006$.
\begin{table}[htb!]
\centering
\small
\caption{General capability and alignment after unlearning (TOFU, Llama-3.2-3B). $|\Delta|/\mathrm{s.e.}$ is the difference over its standard error; IFEval instruction-level scores report no standard error.}
\label{tab:capability}
\begin{tabular}{l|cc|cc}
\toprule
task & SimNPO & SimNPO+CL & $\Delta$ & $|\Delta|/\mathrm{s.e.}$ \\
\midrule
MMLU & 0.5898 & 0.5931 & $+0.003$ & 0.59 \\
GSM8K (exact match) & 0.6444 & 0.6543 & $+0.010$ & 0.53 \\
HumanEval (pass@1) & 0.4512 & 0.4146 & $-0.037$ & 0.67 \\
MBPP (pass@1) & 0.4480 & 0.4420 & $-0.006$ & 0.19 \\
TruthfulQA (mc2) & 0.4394 & 0.4492 & $+0.010$ & 0.44 \\
ARC-Challenge (acc) & 0.4334 & 0.4479 & $+0.015$ & 0.71 \\
ARC-Challenge (acc\_norm) & 0.4420 & 0.4650 & $+0.023$ & 1.12 \\
HellaSwag (acc) & 0.5356 & 0.5419 & $+0.006$ & 0.89 \\
HellaSwag (acc\_norm) & 0.7141 & 0.7180 & $+0.004$ & 0.61 \\
IFEval (prompt, strict) & 0.4824 & 0.4861 & $+0.004$ & 0.12 \\
IFEval (prompt, loose) & 0.5749 & 0.5545 & $-0.020$ & 0.67 \\
IFEval (inst., strict) & 0.6283 & 0.6235 & $-0.005$ & -- \\
IFEval (inst., loose) & 0.6906 & 0.6799 & $-0.011$ & -- \\
\bottomrule
\end{tabular}
\end{table}

\subsubsection{Forget-Set Output Quality}
\label{supp:output}
We generate answers to all $400$ forget-set questions and measure whether the model refuses (pattern-matched refusal phrases), how coherent the output is (the probability of the ``clean'' class under the \texttt{madhurjindal/autonlp-Gibberish-Detector-492513457} classifier), and how repetitive it is (share of the most frequent 4-gram). No configuration ever refuses: the refusal rate is $0.000$ for the base method and for every CLReg setting, so the failure mode on forget queries is degeneration rather than a polite decline; converting it into explicit refusals would require IDK-style forget targets. Because $\lambda$ is the strength of the regularizer, we sweep it to determine whether the degeneration is inherent to CLReg or a property of the operating point (Tab.~\ref{tab:output}). Coherence is governed by the operating point: unlearning quality rises quickly and then plateaus, while coherence is preserved at low $\lambda$ and degrades as the term strengthens. At $\lambda=0.3$, output is \emph{more} coherent than the base method's, with the strongest relearning resistance and a gain of $0.078$ in \texttt{UnlearningScore}. The best-scoring setting is not the most coherent one because \texttt{ForgetScore} rewards not reproducing the forget answer, which degenerate text also achieves; the metric is blind to output usability, so we report both axes. The gains at $\lambda=1$ are nevertheless not an artifact of degeneration: they also appear on \texttt{PrivLeak} and in relearning resistance, neither of which incoherent output would produce. The degradation is confined to forget-set queries (Supp.~\ref{supp:capability}).
\begin{table}[htb!]
\centering
\small
\caption{Forget-set output quality along the regularization path (SimNPO+CL, TOFU Llama-3.2-3B; $\lambda=0$ is SimNPO without CLReg). P(clean) is the gibberish classifier's probability of the ``clean'' class; recovery AUC is the attack of Supp.~\ref{supp:relearn} applied to each arm (seed 0).}
\label{tab:output}
\begin{tabular}{l|ccccc}
\toprule
$\lambda$ & 0 & 0.3 & 0.5 & 1.0 & 2.0 \\
\midrule
P(clean)$\uparrow$ & 0.868 & \textbf{0.887} & 0.781 & 0.585 & 0.481 \\
distinct-token ratio$\uparrow$ & \textbf{0.927} & 0.860 & 0.613 & 0.380 & 0.280 \\
max 4-gram share$\downarrow$ & \textbf{0.056} & 0.076 & 0.228 & 0.474 & 0.618 \\
refusal rate & 0.000 & 0.000 & 0.000 & 0.000 & 0.000 \\
\texttt{UnlearningScore}$\uparrow$ & 0.688 & 0.766 & 0.781 & \textbf{0.785} & 0.780 \\
recovery AUC$\downarrow$ & 0.310 & \textbf{0.212} & 0.220 & 0.241 & 0.251 \\
\bottomrule
\end{tabular}
\end{table}

\subsubsection{Model Size}
\label{supp:size}
Tab.~\ref{tab:size} adds Llama-3.2-1B to the TOFU size range. CLReg helps at every size. The magnitude tracks how mis-calibrated the base method is rather than scale: at 3B the base method under-unlearns markedly ($|\texttt{PrivLeak}|=61.7$) and CLReg corrects it, whereas at 1B and 8B the base methods are closer to balance and the headroom is smaller. The backbone is otherwise fixed by each benchmark's protocol: MUSE releases target and retrained checkpoints only for Llama-2-7B, and WMDP follows the RMU protocol on Zephyr-7B-$\beta$.
\begin{table}[htb!]
\centering
\small
\caption{Model-size ladder on TOFU (\texttt{UnlearningScore}, SimNPO vs.\ SimNPO+CL).}
\label{tab:size}
\begin{tabular}{l|cc|c}
\toprule
model & SimNPO & SimNPO+CL & $\Delta$ \\
\midrule
Llama-3.2-1B & 0.681 & 0.731 & $+0.050$ \\
Llama-3.2-3B & 0.688 & 0.785 & $+0.097$ \\
Llama-3.1-8B & 0.752 & 0.797 & $+0.045$ \\
\bottomrule
\end{tabular}
\end{table}

\subsubsection{When CLReg Degrades Performance}
\label{supp:lambda}
The largest regression in the main tables is the \texttt{PrivLeak} of GradDiff+CL on MUSE-News, which moves from $-87.8$ to $+99.7$. Because $\lambda$ is the strength of the regularizer, sweeping it tells us whether this is an inherent failure of CLReg or the behavior at one operating point (Tab.~\ref{tab:lambda}). It is the operating point: \texttt{PrivLeak} changes sign between $\lambda=0.1$ and $0.3$, and retain quality is essentially untouched up to $\lambda=0.3$. The trade-off is monotone and interpretable: forgetting depth rises with $\lambda$, retain quality falls, and privacy balance passes through zero, along the same axis that governs the number of negatives and $\tau$ (Supp.~\ref{supp:negatives}). The composite score is not monotone: $\lambda=2$ drops it below the no-CLReg baseline, so $\lambda=1$ lies near the upper end of the useful range. We keep $\lambda=1$ as the default, since it needs no per-setting tuning, underlies every result in the main tables, and is required under repeated unlearning requests (Supp.~\ref{supp:sequential}); lowering it towards $0.2$--$0.3$ is appropriate where privacy balance or forget-set output coherence matters more than forgetting depth.
\begin{table}[htb!]
\centering
\small
\caption{$\lambda$ sweep for GradDiff+CL on MUSE-News. The $\lambda=1$ column reports the run in Tab.~\ref{tab:muse-news}; the other columns are re-runs from the public checkpoints.}
\label{tab:lambda}
\begin{tabular}{l|cccccc}
\toprule
$\lambda$ & 0 (no CL) & 0.1 & 0.3 & 0.5 & 1.0 & 2.0 \\
\midrule
\texttt{ForgetScore}$\uparrow$ & 0.552 & 0.553 & 0.693 & 0.758 & 0.898 & 0.916 \\
\texttt{RetainKnowMem}$\uparrow$ & 0.467 & 0.469 & 0.466 & 0.455 & 0.410 & 0.326 \\
\texttt{UnlearningScore}$\uparrow$ & 0.506 & 0.508 & 0.557 & \textbf{0.569} & 0.563 & 0.481 \\
\texttt{PrivLeak}$\rightarrow0$ & $-87.6$ & $-78.4$ & $\mathbf{+25.8}$ & $+89.0$ & $+99.7$ & $+91.5$ \\
\bottomrule
\end{tabular}
\end{table}

\subsection{Computational Overhead}
\label{supp:compute}
\vspace{-5pt}
CLReg introduces only \emph{one extra forward pass} on the paraphrased input per iteration; dropout-based augmentation is applied on cached hidden states at negligible additional cost, and no extra trainable parameters are introduced. The contrastive term can be computed on the already-pooled, $\ell_2$-normalized embeddings. Tab.~\ref{tab:compute} reports wall-clock runtimes (seconds) for $10$ unlearning epochs on TOFU for each base method and its CLReg-augmented counterpart, measured on the same H100 GPU.
\begin{table}[htb!]
\centering
\small
\caption{Wall-clock runtime (seconds) for 10 unlearning epochs on TOFU, with and without CLReg. Overhead is dominated by one extra forward pass on the paraphrased input; InfoNCE and DPO-CL have virtually identical runtime.}
\label{tab:compute}
\resizebox{\linewidth}{!}{%
\begin{tabular}{l|cc|cc|cc|cc|cc}
\toprule
 & GradDiff & +CL & UnDIAL & +CL & NPO & +CL & SimNPO & +CL & PDU & +CL \\
\midrule
Llama-3-3B & 131 & 195 & 147 & 270 & 189 & 252 & 157 & 221 & 130 & 193 \\
Llama-3-8B & 150 & 222 & 169 & 313 & 212 & 284 & 175 & 249 & 152 & 223 \\
\bottomrule
\end{tabular}}
\end{table}
Across most methods the overhead is an additional $\sim$64\,s (3B) / $\sim$72\,s (8B), i.e.\ a relative cost of $34$--$49\%$. UnDIAL+CL has a larger relative overhead, which is an implementation detail -- UnDIAL's hidden states are not natively cached in our code path, requiring an extra full forward pass -- rather than a fundamental cost of CLReg. Because MUSE has no paraphrases, dropout alone provides the contrastive signal there (Tab.~\ref{tab:muse-all}), and the extra paraphrase-forward cost is avoided. CLReg thus trades a modest, predictable compute overhead for consistent improvements across methods and datasets, and does \emph{not} require larger batch sizes or extra parameters.

\subsection{Limitations and Future Work}
\label{supp:limitation}
\vspace{-5pt}
\paragraph{When should CLReg be used, and when might it fail?} CLReg is most useful when (i) the base method suffers from under-unlearning or tangled forget/retain features (large $|\texttt{PrivLeak}|$, low \texttt{ForgetScore}, high entanglement as measured in Tabs.~\ref{tab:tofu-feat}--\ref{tab:muse-feat}), (ii) the forget concept is semantic and later-layer-localized, and (iii) retain-side utility must be preserved. Failure modes include: (a) extremely small forget sets where positive-pair construction is unstable, (b) base methods whose objectives already aggressively collapse forget representations -- CLReg can push further and lead to over-unlearning (e.g., positive \texttt{PrivLeak} for GradDiff+CL on MUSE-News), which Supp.~\ref{supp:lambda} traces to the fixed $\lambda=1$ rather than to the method, and (c) forget concepts encoded in very early layers, for which the default later-layer choice (Sec.~\ref{sec:which-layer}) is mismatched. Mitigations include lowering $\lambda$ (whose effect on forgetting depth, retain quality, privacy balance and output coherence is characterized in Supp.~\ref{supp:lambda} and~\ref{supp:output}), reducing $\tau$, or moving the CL target layer.

\paragraph{Limitations.} \emph{(i) Design sensitivity}: contrastive objectives hinge on choices of augmentation, batch size, and negatives, so CLReg may require dataset-specific tuning (cf.\ Supp.~\ref{supp:cl-ablation} and the per-method settings in Supp.~\ref{supp:exp-settings}). \emph{(ii) Base-method dependence}: as a regularizer, CLReg's efficacy depends on the base unlearning loss and the constructiveness of its interaction with the contrastive term (Tabs.~\ref{tab:tofu-all}--\ref{tab:muse-all}). \emph{(iii) Light, one-step theory}: the analysis in Sec.~\ref{sec:motivation} is a representation-level, one-step argument; it does not establish multi-step convergence, retain-side distortion bounds, or formal memorization-removal / privacy guarantees -- all of which remain open for \emph{every} approximate LLM unlearning method we are aware of (NPO, SimNPO, GradDiff, UnDIAL, PDU, RMU). \emph{(iv) Sequential unlearning}: we study one sequential setting (four MUSE-News requests, Supp.~\ref{supp:sequential}), where CLReg sharply reduces utility erosion but its per-request forgetting weakens along the chain; longer request streams and other benchmarks remain open. \emph{(v) Recovery}: CLReg slows relearning rather than preventing it, and recovery curves with and without CLReg converge within $64$ attack steps (Supp.~\ref{supp:relearn}); like other approximate methods, it does not guarantee removal. \emph{(vi) Seeds}: the main tables report single runs because of the cost of LLM unlearning; we quantify seed variance on a representative pair (Tab.~\ref{tab:seeds}), where the CLReg gain is about thirty times the seed-to-seed standard deviation.

\subsection{Broader Impacts}
\label{supp:impact}
\vspace{-5pt}
Machine unlearning supports privacy and copyright compliance and the removal of hazardous knowledge, as in WMDP, and CLReg aims to make such removal more effective and more durable. Its main risk is a false sense of security: like every approximate unlearning method we are aware of, CLReg offers no formal removal guarantee, and our relearning attack shows that it slows rather than prevents recovery. Unlearned models should therefore not be treated as equivalent to models never trained on the data, and should be evaluated against adversarial recovery before deployment. Unlearning can also be used to suppress legitimate information; CLReg does not change who makes that choice or how it should be governed.

\paragraph{Future work.} (i) \emph{Surgical removal or clipping} of the disentangled forget subspace after CLReg training, building on the clear separation visible in Fig.~\ref{fig:umap} to achieve more faithful unlearning. (ii) \emph{Richer augmentations and negative sampling} beyond dropout and paraphrase. (iii) \emph{Concept-dependent layer selection}: our default later-layer regularization (Sec.~\ref{sec:which-layer}) is a heuristic suited to current benchmarks; low-level (syntactic, stylistic) forget concepts may reside earlier, and probing or attribution could pick the right layer per concept. (iv) \emph{Formal privacy and fairness guarantees} for representation-shaped unlearning, and study of CLReg under repeated, incremental, or continual-learning settings.

\newpage
\section*{NeurIPS Paper Checklist}

\begin{enumerate}

\item {\bf Claims}
    \item[] Question: Do the main claims made in the abstract and introduction accurately reflect the paper's contributions and scope?
    \item[] Answer: \answerYes{} 
    \item[] Justification: See listed contributions in Introduction and check with Method and Experiment Sections.

\item {\bf Limitations}
    \item[] Question: Does the paper discuss the limitations of the work performed by the authors?
    \item[] Answer: \answerYes{} 
    \item[] Justification: See Supp.~\ref{supp:limitation}.

\item {\bf Theory assumptions and proofs}
    \item[] Question: For each theoretical result, does the paper provide the full set of assumptions and a complete (and correct) proof?
    \item[] Answer: \answerYes{} 
    \item[] Justification: The paper contains a light, one-step theory. Each result states its assumptions (Sec.~\ref{sec:motivation}), complete proofs are in Supp.~\ref{supp:proof}, including the multi-negative form (Prop.~\ref{prop:multineg}), and the one non-automatic condition of that form is verified empirically during training.

    \item {\bf Experimental result reproducibility}
    \item[] Question: Does the paper fully disclose all the information needed to reproduce the main experimental results of the paper to the extent that it affects the main claims and/or conclusions of the paper (regardless of whether the code and data are provided or not)?
    \item[] Answer: \answerYes{} 
    \item[] Justification: Detailed experiment setup discussed in Experiment and Supplementary sections.

\item {\bf Open access to data and code}
    \item[] Question: Does the paper provide open access to the data and code, with sufficient instructions to faithfully reproduce the main experimental results, as described in supplemental material?
    \item[] Answer: \answerYes{} 
    \item[] Justification: All data used in this work is open-source, and our code is available at \url{https://github.com/HaoranTang/CLReg}.

\item {\bf Experimental setting/details}
    \item[] Question: Does the paper specify all the training and test details (e.g., data splits, hyperparameters, how they were chosen, type of optimizer) necessary to understand the results?
    \item[] Answer: \answerYes{} 
    \item[] Justification: Detailed experiment setup discussed in Experiment and Supplementary sections.

\item {\bf Experiment statistical significance}
    \item[] Question: Does the paper report error bars suitably and correctly defined or other appropriate information about the statistical significance of the experiments?
    \item[] Answer: \answerYes{} 
    \item[] Justification: The main tables report single runs because of the cost of LLM unlearning. We quantify seed variance over three seeds on a representative pair (Tab.~\ref{tab:seeds}): the CLReg gain in \texttt{UnlearningScore} is $+0.100\pm0.003$, about thirty times the seed-to-seed standard deviation, and the relearning effect is $-0.101\pm0.032$, negative in every seed. Capability results report standard errors (Tab.~\ref{tab:capability}).

\item {\bf Experiments compute resources}
    \item[] Question: For each experiment, does the paper provide sufficient information on the computer resources (type of compute workers, memory, time of execution) needed to reproduce the experiments?
    \item[] Answer: \answerYes{} 
    \item[] Justification: See Supp.~\ref{supp:exp-settings}.
    
\item {\bf Code of ethics}
    \item[] Question: Does the research conducted in the paper conform, in every respect, with the NeurIPS Code of Ethics \url{https://neurips.cc/public/EthicsGuidelines}?
    \item[] Answer: \answerYes{} 
    \item[] Justification: Followed and checked.

\item {\bf Broader impacts}
    \item[] Question: Does the paper discuss both potential positive societal impacts and negative societal impacts of the work performed?
    \item[] Answer: \answerYes{} 
    \item[] Justification: See Supp.~\ref{supp:impact}. Unlearning supports privacy and copyright compliance and hazardous-knowledge removal; its main risk is a false sense of security, since approximate unlearning, ours included, slows rather than prevents recovery.
    
\item {\bf Safeguards}
    \item[] Question: Does the paper describe safeguards that have been put in place for responsible release of data or models that have a high risk for misuse (e.g., pre-trained language models, image generators, or scraped datasets)?
    \item[] Answer: \answerNA{} 
    \item[] Justification: The paper poses no such risks.

\item {\bf Licenses for existing assets}
    \item[] Question: Are the creators or original owners of assets (e.g., code, data, models), used in the paper, properly credited and are the license and terms of use explicitly mentioned and properly respected?
    \item[] Answer: \answerYes{} 
    \item[] Justification: Data is open-source, code is based on cited work and will be released, previous related work properly cited.

\item {\bf New assets}
    \item[] Question: Are new assets introduced in the paper well documented and is the documentation provided alongside the assets?
    \item[] Answer: \answerNA{} 
    \item[] Justification: The paper does not release new assets, just new code which will be released.

\item {\bf Crowdsourcing and research with human subjects}
    \item[] Question: For crowdsourcing experiments and research with human subjects, does the paper include the full text of instructions given to participants and screenshots, if applicable, as well as details about compensation (if any)? 
    \item[] Answer: \answerNA{} 
    \item[] Justification: The paper does not involve crowdsourcing nor research with human subjects.

\item {\bf Institutional review board (IRB) approvals or equivalent for research with human subjects}
    \item[] Question: Does the paper describe potential risks incurred by study participants, whether such risks were disclosed to the subjects, and whether Institutional Review Board (IRB) approvals (or an equivalent approval/review based on the requirements of your country or institution) were obtained?
    \item[] Answer: \answerNA{} 
    \item[] Justification: The paper does not involve crowdsourcing nor research with human subjects.

\item {\bf Declaration of LLM usage}
    \item[] Question: Does the paper describe the usage of LLMs if it is an important, original, or non-standard component of the core methods in this research? Note that if the LLM is used only for writing, editing, or formatting purposes and does \emph{not} impact the core methodology, scientific rigor, or originality of the research, declaration is not required.
    \item[] Answer: \answerYes{} 
    \item[] Justification: An LLM (Llama-3.1-8B-Instruct) generates the paraphrase positives that CLReg uses on TOFU; the generator, prompt, sampling parameters and usage are described in Supp.~\ref{supp:impl}. Otherwise, LLMs are the subject of study rather than tools in the methodology.

\end{enumerate}

\end{document}